\documentclass[12pt]{article}

\usepackage[margin=1in]{geometry}
\usepackage{amsmath}
\usepackage{amssymb}
\usepackage{amsfonts}
\usepackage{amsthm}
\usepackage{aliascnt}
\usepackage{mathtools}
\usepackage{bm}
\usepackage{booktabs}
\usepackage{array}
\usepackage{enumitem}
\usepackage[numbers,sort&compress]{natbib}
\usepackage[hidelinks]{hyperref}
\usepackage[nameinlink,capitalise,noabbrev]{cleveref}

\setlist[enumerate]{leftmargin=.5in}
\setlist[itemize]{leftmargin=.5in}
\allowdisplaybreaks
\numberwithin{equation}{section}

\newtheorem{theorem}{Theorem}[section]

\newaliascnt{lemma}{theorem}
\newtheorem{lemma}[lemma]{Lemma}
\aliascntresetthe{lemma}

\newaliascnt{corollary}{theorem}
\newtheorem{corollary}[corollary]{Corollary}
\aliascntresetthe{corollary}

\newaliascnt{proposition}{theorem}
\newtheorem{proposition}[proposition]{Proposition}
\aliascntresetthe{proposition}

\theoremstyle{definition}
\newaliascnt{definition}{theorem}
\newtheorem{definition}[definition]{Definition}
\aliascntresetthe{definition}

\newaliascnt{assumption}{theorem}
\newtheorem{assumption}[assumption]{Assumption}
\aliascntresetthe{assumption}

\theoremstyle{remark}
\newaliascnt{remark}{theorem}
\newtheorem{remark}[remark]{Remark}
\aliascntresetthe{remark}

\newaliascnt{example}{theorem}
\newtheorem{example}[example]{Example}
\aliascntresetthe{example}

\newaliascnt{counterexample}{theorem}
\newtheorem{counterexample}[counterexample]{Counterexample}
\aliascntresetthe{counterexample}

\crefname{theorem}{Theorem}{Theorems}
\Crefname{theorem}{Theorem}{Theorems}
\crefname{lemma}{Lemma}{Lemmas}
\Crefname{lemma}{Lemma}{Lemmas}
\crefname{corollary}{Corollary}{Corollaries}
\Crefname{corollary}{Corollary}{Corollaries}
\crefname{proposition}{Proposition}{Propositions}
\Crefname{proposition}{Proposition}{Propositions}
\crefname{definition}{Definition}{Definitions}
\Crefname{definition}{Definition}{Definitions}
\crefname{assumption}{Assumption}{Assumptions}
\Crefname{assumption}{Assumption}{Assumptions}
\crefname{remark}{Remark}{Remarks}
\Crefname{remark}{Remark}{Remarks}
\crefname{example}{Example}{Examples}
\Crefname{example}{Example}{Examples}
\crefname{counterexample}{Counterexample}{Counterexamples}
\Crefname{counterexample}{Counterexample}{Counterexamples}

\DeclareMathOperator{\supp}{supp}

\newcommand{\R}{\mathbb{R}}
\newcommand{\E}{\mathbb{E}}

\newcommand{\norm}[1]{\left\lVert #1 \right\rVert}

\newcommand{\eps}{\varepsilon}

\hypersetup{
  pdftitle={Identifiability and Stability of Generative Drifting in the Companion-Elliptic Kernel Family},
  pdfauthor={Hak Geun Lee and Hyonho Chun}
}

\begin{document}

\title{Identifiability and Stability of Generative Drifting\\
in the Companion-Elliptic Kernel Family}
\author{Hak Geun Lee \quad Hyonho Chun\\[6pt]
\normalsize Department of Mathematical Sciences,\\
\normalsize Korea Advanced Institute of Science and Technology (KAIST)}
\date{}

\maketitle

\begin{abstract}
A drifting model is a one-step generator trained by moving each sample along a field of kernel-weighted attraction toward data samples and repulsion between model samples; training halts once this field vanishes. The soundness of this scheme rests on two questions: whether a zero-field equilibrium guarantees agreement with the data distribution, and how the error is controlled when the field is small. We answer both questions. We introduce the companion-elliptic kernel class, which contains the Laplace kernel, and prove that a vanishing field identifies arbitrary Borel probability measures within this class and its blockwise product extension. The scalar companion-elliptic kernels are exactly the Gaussian and Mat\'ern families, and deconvolution, the recovery of a measure from its kernel-smoothed density, is possible precisely when the zero set of the kernel's Fourier transform has empty interior. At approximate equilibria, a mass-escape counterexample shows that field information on a bounded region cannot control the global error. We therefore prove an explicit stability inequality bounding the error by the field residual on the observation region plus the convolution mass outside it. In this inequality the Gaussian kernel requires only field values but amplifies high-frequency errors exponentially, whereas the Mat\'ern kernel requires one additional order of derivative information in exchange for polynomial amplification. Since this gap widens exponentially with frequency, the balance favors the Mat\'ern family and hence the Laplace kernel.
\end{abstract}



\section{Introduction}

\subsection{Generative drifting}

\noindent Drifting, introduced by Deng et al.~\cite{deng2026drifting}, is a training scheme for one-step generative models. Given a prior law $p_{\eps}$ on a latent space $\R^C$ and a generator $f_\theta:\R^C\to\R^d$, the model law is the pushforward $q_\theta=[f_\theta]_\# p_{\eps}$. Diffusion models \cite{ho2020ddpm,song2021sde} and flow matching \cite{lipman2023flowmatching} realize a complex pushforward through multi-step updates at inference time; drifting instead keeps inference single-step and evolves the pushforward law itself during training. The training process thus generates a parameter sequence $\theta_0,\theta_1,\theta_2,\dots$ together with the chain of laws $q_i=[f_{\theta_i}]_\# p_{\eps}$, and this evolution is tracked directly at the distributional level through a vector field $V_{p,q}:\R^d\to\R^d$ depending on the target law $p$ and the current law $q$. Ideally, a generated sample $x_i=f_{\theta_i}(\eps)\sim q_i$ moves according to
\[
x_{i+1}\approx x_i+V_{p,q_i}(x_i),
\]
and the accumulated drift pushes $q_i$ toward $p$. At convergence the two laws must coincide and the field must vanish, so the drifting field is required to satisfy the antisymmetry relation $V_{p,q}=-V_{q,p}$, from which $p=q \Longrightarrow V_{p,q}\equiv 0$ follows. At an abstract level the field is a kernel-mediated interaction between samples, $V_{p,q}(x)=\E_{y^+\sim p,\,y^-\sim q}[K(x,y^+,y^-)]$, where the positive samples $y^+$ are drawn from the data distribution and the negative samples $y^-$ from the model distribution. Its prototypical form is a mean-shift-type attraction--repulsion field \cite{comaniciu2002meanshift,deng2026drifting,fukunaga1975gradient}: for a positive kernel $k$, one sets
\begin{equation}\label{eq:intro-drifting-attraction-repulsion}
\begin{aligned}
V_{p,q}&=V_p^+-V_q^-,\\
V_p^+(x)&:=
\frac{\int k(x,y)(y-x)\,p(dy)}{\int k(x,y)\,p(dy)},\\
V_q^-(x)&:=
\frac{\int k(x,y)(y-x)\,q(dy)}{\int k(x,y)\,q(dy)}.
\end{aligned}
\end{equation}
The Laplace kernel is the default choice of the positive kernel $k$ in \cite{deng2026drifting}.

\subsection{Problem formulation and main results}

\noindent The central question raised by Deng et al.~\cite{deng2026drifting} is whether a zero-drift equilibrium guarantees agreement of the distributions. Subsequent work has established a score identity \cite{cao2026gradientflow,gretton2026wasserstein,lai2026unified,turan2026secretly} and equilibrium identifiability \cite{cao2026gradientflow,kazanskii2026dmf,turan2026secretly} for the Gaussian kernel, but the original drifting field, which takes the Laplace kernel as its default, calls for a separate analysis. Accordingly, this paper studies the following questions for translation-invariant kernels $k(x,y)=\kappa(x-y)$, in particular $\kappa(z)=e^{-\norm{z}/\tau}$ with $\tau>0$ and its generalizations.
\begin{enumerate}[label=(Q\arabic*)]
\item \emph{Identifiability.} Does $V_{p,q}\equiv0$ imply $p=q$ for arbitrary Borel probability measures $p,q$?
\item \emph{Quantitative stability.} When the field is small rather than exactly zero, by what explicit inequality can the distributional error be controlled?
\item \emph{Qualitative stability.} Given a sequence of model distributions $(q_n)$ for which one knows only that $V_{p,q_n}\to0$, with no rate information, under what compactness condition does $q_n$ converge weakly to $p$?
\end{enumerate}

\noindent In brief, our answers are as follows. For (Q1), identifiability holds for arbitrary Borel probability measures in the companion-elliptic class, which contains the Laplace kernel, and in its blockwise product extension (\Cref{thm:companion-elliptic,thm:block-separable-ce-identifiability}); up to positive scalar multiples, the scalar kernels of this class are exactly the Gaussian and Mat\'ern families (\Cref{cor:ce-classification}), and for the deconvolution problem alone, recovering the original distributions from the kernel-smoothed densities, the sharp condition is that the zero set of $\widehat\kappa$ have empty interior (\Cref{prop:identifiability-spectral-sharpness}). For (Q2), field information on a bounded region alone cannot control the global error (\Cref{cex:local-approx-failure}), while an explicit inequality holds once the convolution mass outside the observation region is added (\Cref{thm:explicit-first-stage-stability,cor:explicit-two-stage-identifiability}); the resulting bounds trade the information a kernel demands against the amplification it incurs, favoring the Mat\'ern family (\Cref{rem:two-stage-kernel-conditioning}). For (Q3), tightness suffices: pointwise field convergence then yields weak convergence (\Cref{thm:tight-field-stability}).

\subsection{Related work}\label{subsec:related-work}

\noindent For the Gaussian kernel, it has been shown that the population mean-shift field coincides exactly with the score difference of the Gaussian-smoothed densities \cite{cao2026gradientflow,gretton2026wasserstein,lai2026unified,turan2026secretly,weber2023scoredifference}, an identity that makes explicit the relation between score matching \cite{hyvarinen2005score} and drifting. In the same Gaussian setting, identifiability of the zero-drift equilibrium has also been proved \cite{cao2026gradientflow,kazanskii2026dmf,turan2026secretly}, and the effect of the kernel's high-frequency attenuation on the convergence timescale has been studied \cite{turan2026secretly}.

For non-Gaussian kernels, the original field and the score field do not coincide in general \cite{balasubramanian2026finiteparticle,franz2026nonconservative,lai2026unified}. Lai et al.~\cite{lai2026unified} derived an exact decomposition for general radial kernels and analyzed conditions under which the residual arising for the Laplace kernel becomes small in the high-dimensional regime. He et al.~\cite{he2026sinkhorn} introduced a field in which the one-sided kernel normalization is replaced by a two-sided Sinkhorn scaling, obtaining a Sinkhorn-divergence interpretation and identifiability; this connects with the theory of entropic optimal transport and Sinkhorn divergences \cite{cuturi2013sinkhorn,feydy2019interpolating}. Franz et al.~\cite{franz2026nonconservative} showed that the original normalized radial drift is in general nonconservative, with the Gaussian as the only radial exception, and then established non-Gaussian identifiability for a conservative field employing a sharp normalization. Esteban-Casadevall et al.~\cite{estebancasadevall2026kernelgradient} replaced the Euclidean displacement by a kernel gradient, obtaining identifiability for characteristic kernels together with a smoothed-KL interpretation.

Existing non-Gaussian identifiability results thus generally modify either the normalization or the displacement \cite{estebancasadevall2026kernelgradient,franz2026nonconservative,he2026sinkhorn}. In contrast, this paper retains the original normalization and displacement and directly analyzes identifiability and stability for the Laplace kernel and the companion-elliptic family containing it.

\section{Notation and preliminary facts}\label{sec:preliminaries}

\noindent Throughout, $d\ge1$ is a fixed dimension; unless stated otherwise, the underlying spatial variable is $x\in\R^d$, and an unsubscripted $\norm{\cdot}$ denotes the finite-dimensional Euclidean norm. $C_b(\R^d)$, $C_0(\R^d)$, and $C_c^\infty(\Omega)$ denote, respectively, the spaces of bounded continuous functions, continuous functions vanishing at infinity, and smooth compactly supported functions; $\mathcal D'(\R^d)$ and $\mathcal S'(\R^d)$ are the spaces of distributions and tempered distributions; and the Lebesgue and Sobolev spaces carry their standard notation. $\mathcal P(\R^d)$ denotes the collection of all Borel probability measures on $\R^d$, and $\delta_z\in\mathcal P(\R^d)$ is the Dirac measure at $z$. Weak convergence $q_n\Rightarrow q$ means that $\int f\,dq_n\to\int f\,dq$ for every $f\in C_b(\R^d)$. $B_r$ and $\overline B_r$ denote the open and closed balls of radius $r$ centered at the origin, and subscripts such as $\R_t$, $\R_x^d$, and $L_x^1$ distinguish the time and space variables where needed. The following facts are standard consequences of the theory of convolution and the Fourier transform \cite{folland1999real,rudin1990fourier}; we use them without proof.
\begin{enumerate}[label=(F\arabic*),leftmargin=2.8em,itemsep=0.4em,topsep=0.3em]
\item The Fourier transforms of $f\in L^1(\R^d)$ and of a finite Borel measure $\mu$ are normalized as
\[
\widehat f(\xi):=\int_{\R^d}e^{-i\langle x,\xi\rangle}f(x)\,dx,
\qquad
\widehat\mu(\xi):=\int_{\R^d}e^{-i\langle x,\xi\rangle}\,\mu(dx).
\]
Moreover, if $\widehat f\in L^1(\R^d)$, then the Fourier inversion formula
\[
f(x)=(2\pi)^{-d}\int_{\R^d}e^{i\langle x,\xi\rangle}\widehat f(\xi)\,d\xi
\qquad(\text{a.e. }x\in\R^d)
\]
holds. When convenient, the Fourier transform and its inverse are also written as $\mathcal F$ and $\mathcal F^{-1}$.
\item If $f\in C_0(\R^d)\cap L^1(\R^d)$ and $\mu$ is a finite positive Borel measure, then
\[
f*\mu\in C_0(\R^d)\cap L^1(\R^d),
\qquad
\|f*\mu\|_{L^1}\le \mu(\R^d)\|f\|_{L^1}.
\]
\item For $f\in L^1(\R^d)$, the Riemann--Lebesgue lemma gives $\widehat f\in C_0(\R^d)$. Moreover, if two finite signed Borel measures have the same Fourier transform, then they coincide.
\end{enumerate}

\section{Identifiability}\label{sec:field-identifiability}

\subsection{Identifiability in the companion-elliptic class}\label{subsec:ce-identifiability}

\noindent This section introduces the companion-elliptic class and proves identifiability there through the two stages formalized in \eqref{eq:two-stage-identifiability}, from the vanishing field to the equality of the smoothed densities, and from there to the equality of the measures. \Cref{subsec:ce-examples-classification} then classifies the class, \Cref{subsec:identifiability-sharpness} determines the sharp condition for the deconvolution stage, and \Cref{subsec:block-separable-ce-identifiability} extends the result to blockwise product kernels. We first set down the translation-invariant form of \eqref{eq:intro-drifting-attraction-repulsion} and fix the notation for the kernel convolution, the centered first moment, and the normalized barycenter.

\begin{definition}[Translation-invariant drift]\label{def:general-kernel-drift}
Let $\kappa:\R^d\to(0,\infty)$ be a Borel function and let $r\in\mathcal P(\R^d)$.
Define
\[
\begin{aligned}
u_r^\kappa(x)&:=\int_{\R^d}\kappa(x-y)\,r(dy),\\
D_r^\kappa&:=\Bigl\{x:u_r^\kappa(x)<\infty,\ 
\int\norm{y-x}\kappa(x-y)\,r(dy)<\infty\Bigr\}.
\end{aligned}
\]
For $x\in D_r^\kappa$, set
\[
w_r^\kappa(x):=\int_{\R^d}(y-x)\kappa(x-y)\,r(dy),
\qquad
a_r^\kappa(x):=x+\frac{w_r^\kappa(x)}{u_r^\kappa(x)}.
\]
Moreover, whenever $\int\norm y\kappa(x-y)\,r(dy)<\infty$, we also write
\[
\begin{aligned}
m_r^\kappa(x)&:=\int y\kappa(x-y)\,r(dy),\\
w_r^\kappa(x)&=m_r^\kappa(x)-x u_r^\kappa(x),\\
a_r^\kappa(x)&=\frac{m_r^\kappa(x)}{u_r^\kappa(x)}.
\end{aligned}
\]
For $p,q\in\mathcal P(\R^d)$, the drifting field is defined by
\[
V_{p,q}^\kappa(x):=a_p^\kappa(x)-a_q^\kappa(x),
\qquad x\in D_p^\kappa\cap D_q^\kappa.
\]
When no confusion can arise, the superscript $\kappa$ is omitted.
\end{definition}

\noindent We first examine the Gaussian kernel, which illustrates what one may expect from a general kernel.
For $\kappa_\sigma(z)=\exp(-\norm z^2/2\sigma^2)$ with $\sigma>0$, all of the quantities above are well defined for every Borel probability measure, and differentiation under the integral sign yields
\[
\nabla_x\kappa_\sigma(x-y)=\frac{y-x}{\sigma^2}\kappa_\sigma(x-y),
\qquad
w_r^{\kappa_\sigma}=\sigma^2\nabla u_r^{\kappa_\sigma}.
\]
Consequently,
\begin{equation}\label{eq:gaussian-score-identity-general}
a_r^{\kappa_\sigma}(x)-x=\sigma^2\nabla\log u_r^{\kappa_\sigma}(x).
\end{equation}
Combined with the strict positivity of the convolution density $u_r^{\kappa_\sigma}$ and Fourier uniqueness for finite measures \cite{folland1999real,rudin1990fourier}, the identity \eqref{eq:gaussian-score-identity-general} already yields identifiability for the Gaussian kernel. The same proof, however, cannot be repeated verbatim for the Laplace kernel $\kappa_\tau(z)=e^{-\norm z/\tau}$.
Formally, $\nabla_x\kappa_\tau(x-y)=\tau^{-1}\,\frac{y-x}{\norm{y-x}}\,\kappa_\tau(x-y)$, so the derivative of $u_r^{\kappa_\tau}$ retains only the direction $(y-x)/\norm{y-x}$, whereas the moment $w_r^{\kappa_\tau}$ constituting the field integrates the full displacement $(y-x)$.
The object needed is therefore not the score of $u_r^{\kappa_\tau}$ but a companion potential $\Psi_r=\eta_\tau*r$ recovering the centered first moment, that is, $\tau^2\nabla\Psi_r=w_r^{\kappa_\tau}$. At the kernel level this forces $\nabla\eta_\tau(z)=-(z/\tau^2)\kappa_\tau(z)$, and solving this radial equation gives
\[
\eta_\tau(z)=\left(1+\frac{\norm z}{\tau}\right)e^{-\norm z/\tau},
\qquad
(I-\tau^2\Delta)\eta_\tau=(d+1)\kappa_\tau.
\]
The second identity is verified by a direct computation using the radial Laplacian formula $\Delta f(s)=f''(s)+(d-1)f'(s)/s$; since the Euclidean Laplace kernel is a positive scalar multiple of the Mat\'ern kernel with smoothness $\nu=1/2$ \cite{matern1986spatial,stein1999interpolation}, this identity is also a special case of \Cref{prop:matern-companion-elliptic}.
Upon convolution with $r$, these two relations become
\begin{equation}\label{eq:laplace-companion}
\tau^2\nabla\Psi_r=w_r^{\kappa_\tau},
\qquad
(I-\tau^2\Delta)\Psi_r=(d+1)u_r^{\kappa_\tau},
\end{equation}
and therefore
\[
a_r^{\kappa_\tau}(x)-x
=(d+1)\tau^2\frac{\nabla\Psi_r(x)}{(I-\tau^2\Delta)\Psi_r(x)}.
\]
This is the relation that replaces the Gaussian score identity \eqref{eq:gaussian-score-identity-general}.
Although \eqref{eq:laplace-companion} was derived for the Laplace kernel, the subsequent arguments do not depend on the specific Laplace form and use only the structural relations in \eqref{eq:laplace-companion}.

\begin{assumption}[Companion-elliptic class]\label{ass:companion-elliptic}
For $\kappa:\R^d\to(0,\infty)$ we assume the following.
\begin{enumerate}[label=(CE\arabic*)]
\item\label{ass:ce-kappa}
$\kappa\in C_0(\R^d)\cap L^1(\R^d)\cap L^\infty(\R^d)\cap W^{1,1}(\R^d)$;
\item\label{ass:ce-eta}
there exist constants $c_1,c_2,\lambda_0>0$ and $\lambda_1\ge0$, together with an everywhere positive function $\eta\in C^2(\R^d)$, such that
\[
\eta,\nabla\eta\in C_0(\R^d)\cap L^1(\R^d)\cap L^\infty(\R^d),
\qquad
\nabla^2\eta\in L^\infty(\R^d),
\]
and
\begin{equation}\label{eq:ce-companion}
\nabla\eta(z)=-\frac{z}{c_1}\kappa(z),
\qquad
(\lambda_0 I-\lambda_1\Delta)\eta(z)=c_2\kappa(z)
\qquad(z\in\R^d).
\end{equation}
\end{enumerate}
\end{assumption}

\noindent We refer to the class of functions satisfying this assumption as the ``companion-elliptic class.'' The Gaussian kernel corresponds to the degenerate elliptic case $\lambda_1=0$ with $\eta=\kappa_\sigma$, $c_1=\sigma^2$, $c_2=\lambda_0=1$, while
the Laplace kernel corresponds to the elliptic case $\lambda_1>0$ with companion $\eta_\tau$, $c_1=\tau^2$, $c_2=d+1$, $\lambda_0=1$. Throughout the remainder of this subsection, $\kappa$ is assumed to satisfy \Cref{ass:companion-elliptic}. The next lemma collects the companion identities used in the sequel.

\begin{lemma}[Companion identities for the companion-elliptic class]\label{lem:companion-elliptic-identities}
For $r\in\mathcal P(\R^d)$, set $\Psi_r^\kappa:=\eta*r$.
Then $D_r^\kappa=\R^d$, and the functions $u_r^\kappa,m_r^\kappa,a_r^\kappa$ are well defined and continuous. Moreover, $\Psi_r^\kappa\in C^2(\R^d)$, $\Psi_r^\kappa>0$, and
\[
u_r^\kappa,\ \Psi_r^\kappa,\ \nabla\Psi_r^\kappa\in C_0(\R^d)\cap L^1(\R^d).
\]
Furthermore,
\begin{align}
m_r^\kappa(x)-x u_r^\kappa(x)&=c_1\nabla\Psi_r^\kappa(x),\label{eq:ce-grad-id}\\
(\lambda_0 I-\lambda_1\Delta)\Psi_r^\kappa(x)&=c_2u_r^\kappa(x),\label{eq:ce-elliptic-id}\\
a_r^\kappa(x)-x&=c_1c_2\frac{\nabla\Psi_r^\kappa(x)}{(\lambda_0 I-\lambda_1\Delta)\Psi_r^\kappa(x)}.\label{eq:ce-barycenter-id}
\end{align}
\end{lemma}

\begin{proof}
The boundedness of $\kappa$ gives $u_r(x)<\infty$. Moreover, since
$z\kappa(z)=-c_1\nabla\eta(z)$,
\[
 \int \norm{y-x}\kappa(x-y)\,r(dy)
 \le c_1\|\nabla\eta\|_{L^\infty}<\infty,
\]
so $D_r^\kappa=\R^d$ and $m_r,a_r$ are well defined. Since $\eta,\nabla\eta,\nabla^2\eta$ are bounded and continuous, differentiation under the integral sign yields
\[
\nabla\Psi_r=(\nabla\eta)*r,
\qquad
\nabla^2\Psi_r=(\nabla^2\eta)*r,
\]
whence $\Psi_r\in C^2(\R^d)$. The positivity of $\eta$ together with $r(\R^d)=1$ gives $\Psi_r>0$.

For all $x,y$,
\[
\nabla_x\eta(x-y)=\frac{y-x}{c_1}\kappa(x-y),
\]
and integrating this identity yields \eqref{eq:ce-grad-id}. Convolving $(\lambda_0 I-\lambda_1\Delta)\eta=c_2\kappa$ with $r$ yields \eqref{eq:ce-elliptic-id}, and dividing by the positive function $u_r$ gives \eqref{eq:ce-barycenter-id}.

Finally, applying the convolution statement of \Cref{sec:preliminaries} to $\kappa$, $\eta$, and the components of $\nabla\eta$ gives
\[
 u_r=\kappa*r,
 \qquad \Psi_r=\eta*r,
 \qquad \nabla\Psi_r=(\nabla\eta)*r
 \in C_0(\R^d)\cap L^1(\R^d).
\]
The identity $m_r=xu_r+c_1\nabla\Psi_r$ gives the continuity of $m_r$, and since $u_r$ is continuous and positive, $a_r=m_r/u_r$ is continuous as well.
\end{proof}

With the companion identities above in hand, field identifiability splits into the two stages
\begin{equation}\label{eq:two-stage-identifiability}
V_{p,q}^\kappa\equiv0
\quad\Longrightarrow\quad
u_p^\kappa=u_q^\kappa
\quad\Longrightarrow\quad
p=q.
\end{equation}
The first stage is the problem of deducing the agreement of the convolution densities from $V_{p,q}^\kappa\equiv0$, that is, from the agreement of the normalized barycenters. \Cref{lem:companion-elliptic-identities} provides the first-order relation shared by the two potentials; we convert its Wronskian current into a stationary continuity equation and then fix the ratio of the potentials by \Cref{thm:renorm-sobolev,lem:renorm-mass}. The second stage is the deconvolution problem of recovering the original measures from $u_p^\kappa=u_q^\kappa$, and it uses the positivity $\widehat\kappa>0$ from \Cref{lem:ce-spectral-law} together with Fourier uniqueness. We first assemble the continuity equation tools needed for the first stage.

\begin{proposition}[Renormalization for continuity equations with Sobolev velocity fields]\label{thm:renorm-sobolev}
Let $b\in W^{1,1}_{\mathrm{loc}}(\R^d;\R^d)$, and suppose that
\[
\rho\in L^\infty_{\mathrm{loc}}\bigl(\R_t;L^1(\R^d)\cap L^\infty(\R^d)\bigr)
\]
solves $\partial_t\rho+\nabla\cdot(b\rho)=0$ on $\R_t\times\R^d$ in the distributional sense.
If $\beta\in C^1(\R)$ satisfies $\beta(0)=0$, $\beta'\in L^\infty$, and $z\beta'(z)-\beta(z)\in L^\infty$, then
\[
\partial_t\beta(\rho)+\nabla\cdot\bigl(b\beta(\rho)\bigr)
+\bigl(\beta'(\rho)\rho-\beta(\rho)\bigr)\nabla\cdot b=0
\]
holds in the distributional sense.
In particular, choosing $\beta_\delta(z)=\sqrt{z^2+\delta^2}-\delta$ and passing to the limit shows that $|\rho|$ satisfies $\partial_t|\rho|+\nabla\cdot(b|\rho|)=0$.
\end{proposition}

\begin{proof}
We apply the Sobolev renormalization theorem of DiPerna--Lions \cite{diperna1989transport}; see also Ambrosio's extension to BV vector fields \cite{ambrosio2004transport}. Set $D=(\partial_t,\nabla_x)$ and $B(t,x)=(1,b(x))$. Then
$B\in W^{1,1}_{\mathrm{loc}}(\R^{d+1};\R^{d+1})$ and
$D\cdot B=\nabla_x\cdot b$. Since $\rho$ is locally bounded in space-time and
$\nabla_x\cdot b\in L^1_{\mathrm{loc}}$, the equation can be written as
\[
 B\cdot D\rho=D\cdot(B\rho)-\rho D\cdot B=-\rho\,\nabla_x\cdot b
 \quad\text{in }\mathcal D'(\R^{d+1}).
\]
Applying the cited theorem to the solution $\rho$ and the renormalizing function $\beta$ gives
\[
D\cdot(B\beta(\rho))-\beta(\rho)D\cdot B
=-\beta'(\rho)\rho\,\nabla_x\cdot b,
\]
and expanding the space-time divergence produces the renormalized equation displayed above.

For the absolute value, take
$\beta_\delta(z)=\sqrt{z^2+\delta^2}-\delta$. Then
$\beta_\delta(0)=0$, $|\beta_\delta'|\le1$,
$0\le |z|-\beta_\delta(z)\le\delta$, and
$|z\beta_\delta'(z)-\beta_\delta(z)|\le\delta$. Applying the preceding identity and letting
$\delta\downarrow0$, we see that every term converges in $L^1_{\mathrm{loc}}$. The last term is bounded on compact sets by
$\delta |\nabla_x\cdot b|$. Hence $\partial_t|\rho|+\nabla_x\cdot(b|\rho|)=0$ holds in the distributional sense.
\end{proof}

To complete the Wronskian argument of the first stage, the $L^1$ mass of the renormalized solution must be conserved. The next lemma records this for globally integrable fluxes.

\begin{lemma}[Mass conservation for renormalized integrable fluxes]\label{lem:renorm-mass}
Suppose $b$ and $\rho$ satisfy the assumptions of \Cref{thm:renorm-sobolev}, and in addition $b\rho\in L^\infty_{\mathrm{loc}}\bigl(\R_t;L^1(\R^d;\R^d)\bigr)$.
Then $t\mapsto\|\rho(t,\cdot)\|_{L^1(\R^d)}$ is constant in the distributional sense.
\end{lemma}

\begin{proof}
By \Cref{thm:renorm-sobolev}, $|\rho|$ solves
$\partial_t|\rho|+\nabla\cdot(b|\rho|)=0$. Set $\chi_R(x)=\chi(x/R)$, where $\chi\in C_c^\infty(\R^d)$, $0\le\chi\le1$, and
$\chi\equiv1$ on $B_1$. Testing against $\varphi(t)\chi_R(x)$ gives
\[
-\int \varphi'(t)\int \chi_R|\rho(t)|\,dx\,dt
-\int \varphi(t)\int b|\rho(t)|\cdot\nabla\chi_R\,dx\,dt=0.
\]
The flux term is bounded by
\[
\frac{\|\varphi\|_{L^\infty}\|\nabla\chi\|_{L^\infty}}{R}
\int_{\supp\varphi}\|b\rho(t)\|_{L^1}\,dt,
\]
and hence converges to zero. Since $0\le\chi_R\le1$, $\chi_R\to1$, and
$\rho\in L^\infty_{\mathrm{loc}}(\R_t;L^1_x)$, dominated convergence yields
\[
\int_\R \varphi'(t)\|\rho(t)\|_{L^1}\,dt=0.
\]
Therefore $t\mapsto\|\rho(t)\|_{L^1}$ is distributionally constant.
\end{proof}

This completes the continuity equation toolkit for the first stage. In the second stage, the original measures must be recovered from $u_p^\kappa=u_q^\kappa$. To this end, we show from the companion identities that the Fourier multiplier of the kernel is positive at every frequency.

\begin{lemma}[Spectral law of the companion-elliptic class]\label{lem:ce-spectral-law}
We have $\widehat\kappa\in C^1(\R^d)$ and
\begin{equation}\label{eq:ce-spectral-ode}
\nabla\widehat\kappa(\xi)
=-\frac{c_1c_2}{\lambda_0+\lambda_1\norm\xi^2}\,\xi\,\widehat\kappa(\xi).
\end{equation}
Consequently, $\widehat\kappa$ is radial and everywhere positive.
More explicitly,
\begin{align}
\widehat\kappa(\xi)&=\widehat\kappa(0)\exp\!\left(-\frac{c_1c_2}{2\lambda_0}\norm\xi^2\right),&& \lambda_1=0,\label{eq:ce-spectral-gaussian}\\
\widehat\kappa(\xi)&=\widehat\kappa(0)\left(1+\frac{\lambda_1}{\lambda_0}\norm\xi^2\right)^{-c_1c_2/(2\lambda_1)},&& \lambda_1>0.\label{eq:ce-spectral-matern}
\end{align}
\end{lemma}

\begin{proof}
The identity $z\kappa(z)=-c_1\nabla\eta(z)$ and the assumption
$\nabla\eta\in L^1$ imply $z\kappa\in L^1$. Hence $\widehat\kappa\in C^1$ and $\partial_{\xi_j}\widehat\kappa(\xi)=-i\int z_j e^{-i\langle z,\xi\rangle}\kappa(z)\,dz$.
Taking Fourier transforms of the two companion identities gives
\[
 i\xi\widehat\eta(\xi)=-\frac{i}{c_1}\nabla\widehat\kappa(\xi),
 \qquad
 (\lambda_0+\lambda_1\norm{\xi}^2)\widehat\eta(\xi)=c_2\widehat\kappa(\xi).
\]
Here, when $\lambda_1>0$, the fact that $\Delta\eta=\lambda_1^{-1}(\lambda_0\eta-c_2\kappa)\in L^1$ justifies the second transform as an $L^1$ identity. Eliminating $\widehat\eta$ yields \eqref{eq:ce-spectral-ode}.

For $\omega\in\mathbb S^{d-1}$, set $g_\omega(r)=\widehat\kappa(r\omega)$. Then
\[
 g_\omega'(r)=-\frac{c_1c_2r}{\lambda_0+\lambda_1r^2}g_\omega(r),
 \qquad
 g_\omega(0)=\widehat\kappa(0)=\int\kappa>0.
\]
Solving this scalar ODE gives \eqref{eq:ce-spectral-gaussian} when $\lambda_1=0$ and \eqref{eq:ce-spectral-matern} when $\lambda_1>0$. In either case the right-hand side is independent of $\omega$ and everywhere positive, so $\widehat\kappa$ is radial and everywhere positive.
\end{proof}

We now prove the first implication of \eqref{eq:two-stage-identifiability} using \Cref{lem:companion-elliptic-identities,thm:renorm-sobolev,lem:renorm-mass}, and complete the second implication with \Cref{lem:ce-spectral-law} and Fourier uniqueness.

\begin{theorem}[Field identifiability in the companion-elliptic class]\label{thm:companion-elliptic}
For any $p,q\in\mathcal P(\R^d)$, the field $V_{p,q}^\kappa$ is well defined on all of $\R^d$, and
\[
V_{p,q}^\kappa\equiv0\qquad\Longrightarrow\qquad p=q.
\]
\end{theorem}

\begin{proof}
Write $u_r,m_r,a_r,\Psi_r$ for the objects associated with the fixed kernel.
By \Cref{lem:companion-elliptic-identities}, the field is globally defined.
Assume $V_{p,q}^\kappa\equiv0$. Then $a_p=a_q=:a$, and we set $b(x):=(a(x)-x)/(c_1c_2)$.
For $r=p,q$, \eqref{eq:ce-barycenter-id} and \eqref{eq:ce-elliptic-id} yield
\begin{equation}\label{eq:ce-first-order-relation}
\nabla\Psi_r=b(\lambda_0 I-\lambda_1\Delta)\Psi_r.
\end{equation}
Now set $J:=\Psi_q\nabla\Psi_p-\Psi_p\nabla\Psi_q$. If $\lambda_1=0$, multiplying the two equations in \eqref{eq:ce-first-order-relation} by
$\Psi_q$ and $\Psi_p$ respectively and subtracting gives $J=0$.
We may therefore assume $\lambda_1>0$ and set
\begin{equation}\label{eq:ce-s-definition}
 s:=\nabla\cdot J=\Psi_q\Delta\Psi_p-\Psi_p\Delta\Psi_q.
\end{equation}
The same subtraction yields
\begin{equation}\label{eq:ce-J-bs}
 J=-\lambda_1 b s,
\end{equation}
and therefore
\begin{equation}\label{eq:ce-stationary-s}
 s+\lambda_1\nabla\cdot(bs)=0
\end{equation}
holds in $\mathcal D'(\R^d)$. We now verify the hypotheses of the renormalization argument. First, \eqref{eq:ce-elliptic-id} gives $s=(c_2/\lambda_1)(\Psi_pu_q-\Psi_qu_p)$, so $s\in L^1\cap L^\infty$; and since $\Psi_r,\nabla\Psi_r\in L^1\cap L^\infty$, we have
$J\in L^1(\R^d;\R^d)$. Moreover, $b=\nabla\Psi_q/(c_2u_q)$. Here $\nabla^2\eta\in L^\infty$ implies $\nabla\Psi_q\in W^{1,\infty}_{\mathrm{loc}}$, and
$\kappa\in W^{1,1}$ implies $u_q\in W^{1,1}_{\mathrm{loc}}$. Furthermore, $u_q$ is positive and continuous, hence bounded away from zero on compact sets. The Sobolev chain rule therefore gives
$b\in W^{1,1}_{\mathrm{loc}}$.

To pass from the stationary equation \eqref{eq:ce-stationary-s} to the time-dependent setting of \Cref{thm:renorm-sobolev,lem:renorm-mass}, set $\rho(t,x)=e^{t/\lambda_1}s(x)$; the exponential factor turns \eqref{eq:ce-stationary-s} into the continuity equation $\partial_t\rho+\nabla\cdot(b\rho)=0$. The bounds above and \eqref{eq:ce-J-bs} give
\[
\rho\in L^\infty_{\mathrm{loc}}(\R_t;L^1\cap L^\infty),
\qquad
b\rho=-\frac{e^{t/\lambda_1}}{\lambda_1}J
\in L^\infty_{\mathrm{loc}}(\R_t;L^1(\R^d;\R^d)).
\]
By \Cref{thm:renorm-sobolev,lem:renorm-mass}, $\|\rho(t)\|_{L^1}$ is constant in $t$. But $\|\rho(t)\|_{L^1}=e^{t/\lambda_1}\|s\|_{L^1}$ grows exponentially unless $s$ vanishes, so $s=0$ a.e. Since the expression for $s$ above is continuous, $s\equiv0$, and \eqref{eq:ce-J-bs} gives $J\equiv0$.

In both cases we have thus obtained $J\equiv0$. Since $\Psi_q>0$,
\[
\nabla\left(\frac{\Psi_p}{\Psi_q}\right)=\frac{J}{\Psi_q^2}=0,
\]
and on the connected set $\R^d$ there is a constant $c>0$ with $\Psi_p=c\Psi_q$. Applying \eqref{eq:ce-elliptic-id} gives $u_p=cu_q$. Since for every probability measure $r$
\[
\int_{\R^d}u_r(x)\,dx=\int_{\R^d}\kappa(z)\,dz,
\]
we conclude $c=1$, hence $u_p=u_q$, that is, $\kappa*(p-q)=0$. This proves the first implication of \eqref{eq:two-stage-identifiability}. For the second implication, taking Fourier transforms gives
\[
\widehat\kappa(\xi)(\widehat p(\xi)-\widehat q(\xi))=0.
\]
By \Cref{lem:ce-spectral-law}, $\widehat\kappa>0$, and Fourier uniqueness for the finite signed measure $p-q$ yields $p=q$.
\end{proof}

\subsection{Classification of the companion-elliptic class}
\label{subsec:ce-examples-classification}
\label{subsec:finite-order-closure-sharpness}

\noindent \Cref{thm:companion-elliptic} applies to every kernel satisfying \Cref{ass:companion-elliptic}, and the natural next question is which kernels these are. We first verify that the Mat\'ern family satisfies the companion-elliptic assumption. We then show that generalizing the elliptic relation to a polynomial $P(-\Delta)$ produces no new scalar kernels, and use this to classify the companion-elliptic class.

\begin{lemma}[$W^{1,1}$ regularity]\label{lem:radial-profile-W11}
Let $\phi\in C([0,\infty))\cap C^1((0,\infty))$ satisfy
\[
\int_0^\infty |\phi(r)|r^{d-1}\,dr
+
\int_0^\infty |\phi'(r)|r^{d-1}\,dr
<\infty.
\]
If $f(x):=\phi(\norm x)$, then $f\in W^{1,1}(\R^d)$.
More precisely, for almost every $x\neq0$,
\[
\nabla f(x)=\phi'(\norm x)\frac{x}{\norm x}.
\]
\end{lemma}

\begin{proof}
The polar-coordinate formula gives $f\in L^1$ and
\[
F(x):=\phi'(\norm{x})\frac{x}{\norm{x}}\mathbf 1_{\{x\neq0\}}
\in L^1(\R^d;\R^d).
\]
For $\zeta\in C_c^\infty(\R^d)$ and each coordinate $j$, we integrate by parts on $\{\norm{x}>\varepsilon\}$ when $d\ge2$. Since $\phi$ is continuous at $0$, the boundary term is
\[
-\int_{\norm{x}=\varepsilon}\phi(\varepsilon)\zeta(x)\frac{x_j}{\varepsilon}\,dS(x)=O(\varepsilon^{d-1})\to0.
\]
Letting $\varepsilon\downarrow0$ yields
$\int f\partial_j\zeta=-\int F_j\zeta$. In one dimension, the same argument on
$(-\infty,-\varepsilon)\cup(\varepsilon,\infty)$ produces the boundary term
$\phi(\varepsilon)(\zeta(-\varepsilon)-\zeta(\varepsilon))\to0$. Hence
$F$ is the weak gradient of $f$, and $f\in W^{1,1}(\R^d)$.
\end{proof}

\begin{proposition}[Companion-elliptic structure of the Mat\'ern kernel]\label{prop:matern-companion-elliptic}
Fix $\ell>0$ and $\nu>0$.
Let $K_\nu$ be the modified Bessel function of the second kind, and define
\begin{equation}\label{eq:matern-kernel-definitions}
\kappa_{\nu,\ell}(z):=\ell\norm z^\nu K_\nu\!\left(\frac{\norm z}{\ell}\right),
\qquad
\eta_{\nu,\ell}(z):=\norm z^{\nu+1}K_{\nu+1}\!\left(\frac{\norm z}{\ell}\right),
\end{equation}
where the values at $z=0$ are defined by continuous extension.
Then \Cref{ass:companion-elliptic} holds.
\end{proposition}

\begin{proof}
For $r>0$, set
\[
\phi(r)=\ell r^\nu K_\nu(r/\ell),
\qquad
\psi(r)=r^{\nu+1}K_{\nu+1}(r/\ell).
\]
The Bessel recurrence and derivative identities \cite[\S10.29(i)--(ii)]{dlmf}
\[
\frac{d}{ds}\bigl(s^\mu K_\mu(s)\bigr)=-s^\mu K_{\mu-1}(s),
\qquad
K_{\mu+1}(s)=K_{\mu-1}(s)+\frac{2\mu}{s}K_\mu(s)
\]
yield
\begin{equation}\label{eq:matern-eta-prime-short}
\psi'(r)=-\frac{r}{\ell^2}\phi(r),
\end{equation}
and
\begin{equation}\label{eq:matern-eta-kappa-relation-short}
\psi(r)=-r\phi'(r)+2\nu\phi(r).
\end{equation}
Hence, on $\R^d\setminus\{0\}$,
\[
\nabla\eta_{\nu,\ell}(z)=-\frac{z}{\ell^2}\kappa_{\nu,\ell}(z),
\]
and the radial Laplacian formula gives
\[
\Delta\eta_{\nu,\ell}(z)
=-\frac d{\ell^2}\phi(r)-\frac r{\ell^2}\phi'(r),
\qquad r=\norm{z}.
\]
Combining this with \eqref{eq:matern-eta-kappa-relation-short} yields, on $\R^d\setminus\{0\}$,
\[
(I-\ell^2\Delta)\eta_{\nu,\ell}=(d+2\nu)\kappa_{\nu,\ell}.
\]
It remains to verify the regularity assumptions and to extend the identities to the origin. The small- and large-argument asymptotics of $K_\mu$ \cite[\S\S10.30, 10.40]{dlmf} show that
$\phi$ and $\psi$ have finite positive limits at $0$ and decay exponentially at infinity. After continuous extension at the origin, we therefore have $\kappa_{\nu,\ell},\eta_{\nu,\ell}\in C_0\cap L^1\cap L^\infty$, both everywhere positive.
Moreover,
\[
\phi'(r)=-r^\nu K_{\nu-1}(r/\ell),
\]
and the same asymptotics give
\[
\int_0^\infty |\phi'(r)|r^{d-1}\,dr<\infty,
\qquad
\int_0^\infty |\phi(r)|r^{d-1}\,dr<\infty.
\]
By \Cref{lem:radial-profile-W11}, $\kappa_{\nu,\ell}\in W^{1,1}(\R^d)$.

As for $\eta_{\nu,\ell}$, \eqref{eq:matern-eta-prime-short} gives
\[
\frac{\psi'(r)}{r}=-\frac{\phi(r)}{\ell^2},
\qquad
\psi''(r)=-\frac{\phi(r)+r\phi'(r)}{\ell^2}.
\]
Since $\phi(r)\to\phi(0)$ and $r\phi'(r)\to0$, and both terms decay exponentially at infinity, $\psi'(r)/r$ and $\psi''(r)$ are bounded and have the common limit
$-\phi(0)/\ell^2$ as $r\downarrow0$. The radial Hessian formula
\[
\partial_{ij}\eta_{\nu,\ell}(z)
=\psi''(r)\frac{z_i z_j}{r^2}
+\frac{\psi'(r)}{r}\left(\delta_{ij}-\frac{z_i z_j}{r^2}\right)
\]
therefore extends continuously to $z=0$ with value
$-\phi(0)\delta_{ij}/\ell^2$. Moreover, $|\psi'(r)|\le Cr$ near the origin, so
$\nabla\eta_{\nu,\ell}(0)=0$. Hence
$\eta_{\nu,\ell}\in C^2(\R^d)$ and $\nabla^2\eta_{\nu,\ell}\in L^\infty$.
Finally,
$\nabla\eta_{\nu,\ell}(z)=-(z/\ell^2)\kappa_{\nu,\ell}(z)$ is continuous, integrable, and bounded, and it vanishes at infinity. The two companion identities extend to $z=0$ by continuity, so all conditions of \Cref{ass:companion-elliptic} hold with the constants specified above.
\end{proof}

\noindent It is natural to ask whether replacing the elliptic operator in the companion identity by one of higher order could produce new kernel families. The next theorem shows that even under a positive finite-order relation, only an affine operator is possible.

\begin{theorem}[Order constraint on the companion relation]
\label{thm:finite-order-closure-sharpness}
Let $\kappa:\R^d\to(0,\infty)$ and a real-valued function $\eta$ satisfy
\[
\kappa\in C_0(\R^d)\cap L^1(\R^d)\cap L^\infty(\R^d),
\qquad
\eta\in W^{1,1}(\R^d).
\]
Assume there are $c_1,c_2>0$ and a real polynomial $P\in\R[s]$ with
\[
P(s)>0\qquad(s\ge0),
\]
such that
\begin{equation}\label{eq:finite-order-companion-closure}
\nabla\eta(z)=-\frac{z}{c_1}\kappa(z)
\quad\text{a.e.},
\qquad
P(-\Delta)\eta=c_2\kappa
\quad\text{in }\mathcal D'(\R^d).
\end{equation}
Then $\deg P\le1$. More precisely, only the following two cases can occur.
\begin{enumerate}[label=(\roman*)]
\item $P(s)=\lambda_0$ with $\lambda_0>0$, and $\kappa$ is a positive scalar multiple of a Gaussian kernel.
\item $P(s)=\lambda_0+\lambda_1s$ with $\lambda_0,\lambda_1>0$, and $\kappa$ is a positive scalar multiple of the Mat\'ern kernel $\kappa_{\nu,\ell}$ with
\[
\ell^2=\frac{\lambda_1}{\lambda_0},
\qquad
\nu=\frac{c_1c_2}{2\lambda_1}-\frac d2>0.
\]
\end{enumerate}
\end{theorem}

\begin{proof}
The first identity in \eqref{eq:finite-order-companion-closure} together with $\eta\in W^{1,1}$ gives
\[
z\kappa(z)=-c_1\nabla\eta(z)
\in L^1(\R^d;\R^d).
\]
Hence $\widehat\kappa\in C^1(\R^d)$, and taking Fourier transforms,
\[
i\xi\widehat\eta(\xi)
=-\frac{i}{c_1}\nabla\widehat\kappa(\xi),
\qquad
P(\norm\xi^2)\widehat\eta(\xi)
=c_2\widehat\kappa(\xi).
\]
Consider the second identity. Since $\eta,\kappa\in L^1$, both sides of \eqref{eq:finite-order-companion-closure} are tempered distributions, and since $C_c^\infty(\R^d)$ is dense in $\mathcal S(\R^d)$, the identity holds in $\mathcal S'(\R^d)$. After the Fourier transform, both sides are continuous functions, so the second identity may be read pointwise.
Since $P>0$, eliminating $\widehat\eta$ yields
\begin{equation}\label{eq:finite-order-spectral-ode}
\nabla\widehat\kappa(\xi)
=-\frac{c_1c_2}{P(\norm\xi^2)}\,
\xi\widehat\kappa(\xi).
\end{equation}

For $\omega\in\mathbb S^{d-1}$, set $g_\omega(r):=\widehat\kappa(r\omega)$.
Since $\kappa>0$ and $\kappa\in L^1$,
$g_\omega(0)=\widehat\kappa(0)=\int\kappa>0$.
Solving \eqref{eq:finite-order-spectral-ode} along rays gives
\begin{equation}\label{eq:finite-order-spectral-solution}
g_\omega(r)
=\widehat\kappa(0)
\exp\left(
-\frac{c_1c_2}{2}
\int_0^{r^2}\frac{ds}{P(s)}
\right),
\qquad r\ge0.
\end{equation}
In particular, $\widehat\kappa$ is radial and everywhere positive.
On the other hand, $\kappa\in L^1$ implies $g_\omega(r)\to0$ as $r\to\infty$ by the Riemann--Lebesgue lemma.
Hence
\begin{equation}\label{eq:reciprocal-polynomial-divergence}
\int_0^\infty\frac{ds}{P(s)}=\infty
\end{equation}
is necessary.

Let $N:=\deg P$.
Since $P$ is positive on $[0,\infty)$, its leading coefficient is positive, and
$P(s)\sim a_Ns^N$ as $s\to\infty$.
If $N\ge2$, then $1/P(s)=O(s^{-N})$, contradicting \eqref{eq:reciprocal-polynomial-divergence}.
Hence $N\le1$.

If $N=0$, then $P(s)=\lambda_0>0$ and \eqref{eq:finite-order-spectral-solution} gives
\[
\widehat\kappa(\xi)
=\widehat\kappa(0)
\exp\left(-\frac{c_1c_2}{2\lambda_0}\norm\xi^2\right).
\]
By Fourier uniqueness, $\kappa$ is a positive scalar multiple of a Gaussian kernel.

If $N=1$, then
$P(s)=\lambda_0+\lambda_1s$ with $\lambda_0,\lambda_1>0$, and
\begin{equation}\label{eq:finite-order-matern-spectrum}
\widehat\kappa(\xi)
=\widehat\kappa(0)
\left(1+\ell^2\norm\xi^2\right)^{-\beta},
\qquad
\ell^2:=\frac{\lambda_1}{\lambda_0},
\qquad
\beta:=\frac{c_1c_2}{2\lambda_1}.
\end{equation}
Necessarily $\beta>d/2$.
Indeed, let $\varphi_\varepsilon$ be a Gaussian approximate identity; then
$|\kappa*\varphi_\varepsilon(0)|\le\|\kappa\|_{L^\infty}$.
On the other hand, Fourier inversion for rapidly decaying products together with \eqref{eq:finite-order-matern-spectrum} gives, for some $C_F>0$,
\[
\kappa*\varphi_\varepsilon(0)
=C_F\widehat\kappa(0)
\int_{\R^d}
(1+\ell^2\norm\xi^2)^{-\beta}
 e^{-\varepsilon^2\norm\xi^2/2}\,d\xi.
\]
If $\beta\le d/2$, the integral on the right is bounded below by
$C\int_1^{1/\varepsilon}r^{d-1-2\beta}\,dr$, which diverges as $\varepsilon\downarrow0$.
This contradicts the $L^\infty$ bound above, so $\beta>d/2$.
Hence $\nu:=\beta-d/2>0$, and by the Bessel-potential inversion formula \cite[Theorem~6.13]{wendland2005scattered} and Fourier uniqueness, $\kappa$ is a positive scalar multiple of $\kappa_{\nu,\ell}$. Finally, that both cases are realized follows from the preceding results.
\end{proof}

\noindent Combining the theorem above with the proposition on the Mat\'ern kernel determines the scalar companion-elliptic class completely.

\begin{corollary}[Classification of the companion-elliptic class]\label{cor:ce-classification}
Up to positive scalar multiples, every companion-elliptic kernel is Gaussian or Mat\'ern.
More precisely, the following hold.
\begin{enumerate}[label=(\roman*)]
\item If $\lambda_1=0$, then $\kappa$ is a Gaussian kernel.
\item If $\lambda_1>0$, then $\kappa$ is a Mat\'ern kernel $\kappa_{\nu,\ell}$ for some $\ell>0$ and $\nu>0$.
\end{enumerate}
In particular, every companion-elliptic kernel is radial.
\end{corollary}

\begin{proof}
Applying \Cref{thm:finite-order-closure-sharpness} with $P(s)=\lambda_0+\lambda_1s$ shows that every companion-elliptic kernel is a positive scalar multiple of a Gaussian or Mat\'ern kernel. The converse inclusion follows from the Gaussian identity and \Cref{prop:matern-companion-elliptic}.
\end{proof}

\begin{remark}[Why the Wronskian argument requires a second-order operator]
\label{rem:wronskian-closure-sharpness}
\Cref{thm:finite-order-closure-sharpness} says that passing to a higher-order companion relation does not enlarge the family of admissible kernels.
Independently of this, one can observe that the Wronskian step at the heart of the proof of \Cref{thm:companion-elliptic} is likewise tied exactly to operators of order two.
Indeed, suppose some kernel-companion pair satisfies \eqref{eq:finite-order-companion-closure} and that the analogue of \Cref{lem:companion-elliptic-identities} holds.
If $V_{p,q}^\kappa\equiv0$, then $a_p=a_q=:a$, and, with $b:=(a-x)/(c_1c_2)$, the two identities give
\[
\nabla\Psi_r=bP(-\Delta)\Psi_r,
\qquad r=p,q.
\]
Hence
\[
J:=\Psi_q\nabla\Psi_p-\Psi_p\nabla\Psi_q
=
b\bigl[
\Psi_qP(-\Delta)\Psi_p-\Psi_pP(-\Delta)\Psi_q
\bigr].
\]
Now if $P(s)=\lambda_0+\lambda_1s$, then
\[
\Psi_qP(-\Delta)\Psi_p-\Psi_pP(-\Delta)\Psi_q
=-\lambda_1\nabla\cdot J,
\]
so $J=-\lambda_1b\,\nabla\cdot J$.
Consequently, $s:=\nabla\cdot J$ satisfies
\[
s+\lambda_1\nabla\cdot(bs)=0,
\]
which is precisely the stationary continuity equation handed to the renormalization argument in \Cref{thm:companion-elliptic}. The crux of this reduction is that the bracket term $\Psi_qP(-\Delta)\Psi_p-\Psi_pP(-\Delta)\Psi_q$ can be expressed exactly as a constant multiple of the divergence of the Wronskian current $J$, and this is possible only for affine $P$.
Indeed, suppose that for some constant $\gamma$ every pair of smooth functions $\phi,\psi$ satisfies
\[
\psi P(-\Delta)\phi-\phi P(-\Delta)\psi
=\gamma\nabla\cdot(\psi\nabla\phi-\phi\nabla\psi).
\]
Substituting $\psi\equiv1$ and $\phi(x)=\cos\langle\xi,x\rangle$ gives
\[
P(\norm\xi^2)-P(0)=-\gamma\norm\xi^2,
\]
hence $P(s)=P(0)-\gamma s$.
In other words, the argument that converts the first-order relation supplied by the drift into a continuity equation via the Wronskian current works only when $P$ is affine.
\end{remark}

\subsection{Sharp identifiability condition for deconvolution}
\label{subsec:identifiability-sharpness}

\Cref{thm:finite-order-closure-sharpness,rem:wronskian-closure-sharpness} show that the structure used for the first implication is highly restrictive. The second implication is independent of the companion-elliptic identities and rests on whether the convolution operator can distinguish finite measures. The next theorem gives the sharp Fourier condition for this.

\begin{theorem}[Sharp identifiability condition for deconvolution]
\label{prop:identifiability-spectral-sharpness}
Let $\kappa:\R^d\to(0,\infty)$ satisfy
\[
\kappa\in C_0(\R^d)\cap L^1(\R^d),
\qquad
z\mapsto z\kappa(z)
\in C_0(\R^d;\R^d)\cap L^1(\R^d;\R^d),
\]
and set
\[
Z_\kappa:=\{\xi\in\R^d:\widehat\kappa(\xi)=0\}.
\]
Then the following hold.
\begin{enumerate}[label=(\roman*)]
\item\label{item:spectral-sharpness-injective}
If $\operatorname{int}Z_\kappa=\varnothing$, then $\kappa$-convolution is injective on finite signed Borel measures. That is, for every finite signed Borel measure $\mu$,
\[
\kappa*\mu=0
\quad\Longrightarrow\quad
\mu=0.
\]
In particular, for any $p,q\in\mathcal P(\R^d)$, $u_p^\kappa=u_q^\kappa$ implies $p=q$.
\item\label{item:spectral-sharpness-counterexample}
If $\operatorname{int}Z_\kappa\ne\varnothing$, then there exist distinct $p,q\in\mathcal P(\R^d)$ with bounded, everywhere positive $C^\infty$ densities such that
\[
u_p^\kappa=u_q^\kappa,
\qquad
w_p^\kappa=w_q^\kappa
\qquad\text{on }\R^d.
\]
In particular, $V_{p,q}^\kappa\equiv0$ while $p\ne q$.
\end{enumerate}
\end{theorem}

\begin{proof}
First, let $\mu$ be a finite signed Borel measure with $\kappa*\mu=0$.
Taking Fourier transforms gives
\[
\widehat\kappa(\xi)\widehat\mu(\xi)=0
\qquad(\xi\in\R^d).
\]
Hence $\widehat\mu=0$ on $\R^d\setminus Z_\kappa$.
If $\operatorname{int}Z_\kappa=\varnothing$, this complement is dense, and since the Fourier transform of a finite signed measure is continuous, $\widehat\mu\equiv0$.
By Fourier uniqueness, $\mu=0$. In particular, taking $\mu=p-q$, $u_p^\kappa=u_q^\kappa$ implies $p=q$.

Now suppose $\operatorname{int}Z_\kappa\ne\varnothing$.
Since $\kappa>0$ and $\kappa\in L^1$,
$\widehat\kappa(0)=\int\kappa>0$.
Moreover, since $\kappa$ is real-valued, $Z_\kappa$ is symmetric about the origin.
We may therefore choose an open ball $B$ with
\[
\overline B\subset\operatorname{int}Z_\kappa,
\qquad
B\cap(-B)=\varnothing.
\]
Pick a nonzero real function $\psi\in C_c^\infty(B)$ and set
\[
\varphi(\xi):=\psi(\xi)+\psi(-\xi),
\qquad
g:=\mathcal F^{-1}\varphi.
\]
Then $\varphi$ is a real-valued even function with
\[
0\ne\varphi\in C_c^\infty(\operatorname{int}Z_\kappa),
\qquad
\varphi(0)=0.
\]
Hence $g$ is a real-valued nonzero Schwartz function with $\int g=0$.

Fix $N>d/2$ and set
\[
h(x):=c_N(1+\norm x^2)^{-N},
\qquad
\int_{\R^d}h(x)\,dx=1.
\]
Since $g$ is a Schwartz function,
\[
M:=\sup_{x\in\R^d}\frac{|g(x)|}{h(x)}\in(0,\infty).
\]
Choose $0<\varepsilon<(2M)^{-1}$ and define
\[
p(dx):=(h(x)+\varepsilon g(x))\,dx,
\qquad
q(dx):=(h(x)-\varepsilon g(x))\,dx.
\]
Both densities dominate $h/2$, are smooth and bounded, and integrate to one.
Moreover, $g\ne0$ implies $p\ne q$.

On the other hand, since $z\kappa(z)\in L^1$, we have $\widehat\kappa\in C^1$ and, for each coordinate $j$,
\[
\widehat{z_j\kappa}(\xi)=i\,\partial_{\xi_j}\widehat\kappa(\xi).
\]
Since $\operatorname{supp}\varphi$ lies inside an open set on which $\widehat\kappa$ vanishes identically,
\[
\widehat\kappa\,\varphi=0,
\qquad
(\partial_{\xi_j}\widehat\kappa)\varphi=0
\quad(j=1,\ldots,d).
\]
Fourier uniqueness therefore gives
\[
\kappa*g=0,
\qquad
(z\kappa)*g=0.
\]
These convolutions are continuous, so the equalities hold at every $x\in\R^d$.
Consequently,
\[
u_p^\kappa-u_q^\kappa=2\varepsilon\,\kappa*g=0,
\]
and using the change of variable $z=x-y$,
\[
w_p^\kappa-w_q^\kappa
=-2\varepsilon\,(z\kappa)*g=0.
\]
By assumption, $\kappa$ and $z\kappa$ are bounded, so $D_p^\kappa=D_q^\kappa=\R^d$, and
$u_p^\kappa=u_q^\kappa>0$ together with $w_p^\kappa=w_q^\kappa$ yields
$a_p^\kappa=a_q^\kappa$.
Hence $V_{p,q}^\kappa\equiv0$.
\end{proof}

The next example shows that the Fourier transform of a kernel can vanish on an open set even when the kernel itself is everywhere positive.

\begin{example}[A positive kernel whose Fourier transform vanishes on an open set]
\label{ex:open-spectral-gap-kernel}
Let $\operatorname{sinc}(0)=1$ and
\[
\operatorname{sinc}(t)=\frac{\sin t}{t},
\qquad t\neq 0.
\]
First, in $d=1$, define
\[
\kappa_{\mathrm{gap}}(x)
:=
\operatorname{sinc}^4(x)
+
\operatorname{sinc}^4(\sqrt{2}\,x).
\]
This kernel is positive at every $x\in\R$. Since
$\operatorname{sinc}(x)=O(|x|^{-1})$ as $|x|\to\infty$,
\[
\kappa_{\mathrm{gap}}(x)=O(|x|^{-4}).
\]
Hence
\[
\kappa_{\mathrm{gap}},
\quad
x\kappa_{\mathrm{gap}}
\in C_0(\R)\cap L^1(\R).
\]

We now examine the Fourier transform. $\widehat{\operatorname{sinc}}$ is
a positive constant multiple of the indicator function of the interval $[-1,1]$,
understood in the sense of tempered distributions
\cite{folland1999real,grafakos2014classical}.
Since multiplication of functions corresponds to convolution on the Fourier side,
$\widehat{\operatorname{sinc}^4}$ is
a positive constant multiple of the convolution of four interval indicators.
It is therefore nonnegative with support contained in $[-4,4]$.
For the same reason,
$\widehat{\operatorname{sinc}^4(\sqrt{2}\,\cdot)}$ is nonnegative with
support contained in $[-4\sqrt{2},4\sqrt{2}]$.
Consequently,
\[
\widehat{\kappa_{\mathrm{gap}}}(\xi)\geq 0,
\qquad
\operatorname{supp}\widehat{\kappa_{\mathrm{gap}}}
\subseteq[-4\sqrt{2},4\sqrt{2}].
\]
In particular, $\kappa_{\mathrm{gap}}$ is positive definite \cite{rudin1990fourier}, and
\[
\widehat{\kappa_{\mathrm{gap}}}(\xi)=0,
\qquad |\xi|>4\sqrt{2}.
\]
Hence $Z_{\kappa_{\mathrm{gap}}}$ contains a nonempty open set.

Consequently, by \Cref{prop:identifiability-spectral-sharpness}, there exist
distinct smooth positive densities $p$ and $q$ with
\[
V_{p,q}^{\kappa_{\mathrm{gap}}}\equiv 0.
\]
In dimension $d\geq2$, it suffices to take
\[
\kappa_{\mathrm{gap}}^{(d)}(x)
:=
\kappa_{\mathrm{gap}}(x_1)
\exp\left(-\sum_{j=2}^d x_j^2\right).
\]
This kernel is everywhere positive, and
its Fourier transform is the product of
$\widehat{\kappa_{\mathrm{gap}}}(\xi_1)$ and
a Gaussian Fourier multiplier in the remaining variables.
Hence
\[
\widehat{\kappa_{\mathrm{gap}}^{(d)}}(\xi)=0
\qquad\text{whenever}\qquad
|\xi_1|>4\sqrt{2},
\]
so $Z_{\kappa_{\mathrm{gap}}^{(d)}}$ contains an open set in higher dimensions as well,
and the same non-identifiability conclusion follows.
\end{example}

\subsection{Identifiability in the block-separable companion-elliptic class}
\label{subsec:block-separable-ce-identifiability}

\noindent Finally, we return to the companion structure of the first stage and treat its extension. The scalar classification stems from the requirement that a single potential recover the centered first moment in every direction. Assigning a separate companion potential to each block extends the same reasoning to product kernels. Since the Fourier multiplier of each factor is everywhere positive, the multiplier of the product kernel is everywhere positive as well; in particular, the sharp identifiability condition $\operatorname{int}Z_\kappa=\varnothing$ of \Cref{prop:identifiability-spectral-sharpness} holds.

\begin{assumption}[Block-separable companion-elliptic kernel class]\label{ass:block-separable-ce}
Let $\kappa:\R^d\to(0,\infty)$ be a kernel. We assume the following.
\begin{enumerate}[label=(BCE\arabic*)]
\item\label{ass:bce-decomposition}
There exist an integer $M\ge1$ and integers $d_1,\ldots,d_M\ge1$ with $d_1+\cdots+d_M=d$, and the ambient space is identified as
\[
\R^d=\R^{d_1}\times\cdots\times\R^{d_M},
\qquad
z=(z^{(1)},\ldots,z^{(M)}),\quad z^{(\alpha)}\in\R^{d_\alpha}.
\]
Here $\nabla_\alpha$ and $\Delta_\alpha$ denote the gradient and Laplacian with respect to the $\alpha$-th block variable $z^{(\alpha)}$.

\item\label{ass:bce-factor-ce}
For each $\alpha=1,\ldots,M$ there exists a factor kernel
$\kappa_\alpha:\R^{d_\alpha}\to(0,\infty)$ satisfying on $\R^{d_\alpha}$ the same conditions as in \Cref{ass:companion-elliptic}.
Its companion and constants are denoted by
\[
\eta_\alpha,
\qquad
c_{1,\alpha},c_{2,\alpha},\lambda_{0,\alpha}>0,
\qquad
\lambda_{1,\alpha}\ge0.
\]
Thus, for all $z^{(\alpha)}\in\R^{d_\alpha}$,
\begin{equation}\label{eq:block-factor-companion}
\begin{aligned}
\nabla_\alpha\eta_\alpha(z^{(\alpha)})
&=-\frac{z^{(\alpha)}}{c_{1,\alpha}}\kappa_\alpha(z^{(\alpha)}),\\
(\lambda_{0,\alpha}I-\lambda_{1,\alpha}\Delta_\alpha)
\eta_\alpha(z^{(\alpha)})
&=c_{2,\alpha}\kappa_\alpha(z^{(\alpha)}).
\end{aligned}
\end{equation}

\item\label{ass:bce-product}
The full kernel $\kappa$ is the product of the factors:
\begin{equation}\label{eq:block-product-kernel}
\kappa(z):=\prod_{\alpha=1}^M\kappa_\alpha(z^{(\alpha)}).
\end{equation}
Moreover, for each block $\alpha$, the full-space companion is defined by
\begin{equation}\label{eq:block-companion-def}
\eta^{[\alpha]}(z)
:=\eta_\alpha(z^{(\alpha)})
\prod_{\substack{\beta=1\\ \beta\ne\alpha}}^M
\kappa_\beta(z^{(\beta)}).
\end{equation}
\end{enumerate}
A kernel satisfying this assumption is called a block-separable companion-elliptic kernel.
\end{assumption}

The essential point of this assumption is that $\eta^{[\alpha]}$ is a function on the full space, using the companion $\eta_\alpha$ in the $\alpha$-th block and the original kernel factors in the remaining blocks. At the kernel level, we therefore have
\begin{equation}\label{eq:block-kernel-identities}
\nabla_\alpha\eta^{[\alpha]}(z)
=-\frac{z^{(\alpha)}}{c_{1,\alpha}}\kappa(z),
\qquad
(\lambda_{0,\alpha}I-\lambda_{1,\alpha}\Delta_\alpha)
\eta^{[\alpha]}(z)
=c_{2,\alpha}\kappa(z).
\end{equation}
These are the blockwise analogues of \Cref{eq:ce-companion}. Throughout the remainder of this subsection, $\kappa$ is assumed to satisfy \Cref{ass:block-separable-ce}.

\begin{lemma}[Companion identities for the block-separable companion-elliptic class]\label{lem:block-ce-identities}
For $r\in\mathcal P(\R^d)$, set
\[
u_r^\kappa:=\kappa*r,
\qquad
\Psi_{\alpha,r}^\kappa:=\eta^{[\alpha]}*r,
\qquad
\alpha=1,\ldots,M.
\]
Then $D_r^\kappa=\R^d$, and $u_r^\kappa,m_r^\kappa,a_r^\kappa$ are continuous on $\R^d$.
Furthermore,
\[
u_r^\kappa\in C_0(\R^d)\cap L^1(\R^d)\cap L^\infty(\R^d)
\cap W^{1,1}_{\mathrm{loc}}(\R^d).
\]
For each $\alpha=1,\ldots,M$,
\[
\Psi_{\alpha,r}^\kappa,
\nabla_\alpha\Psi_{\alpha,r}^\kappa
\in C_0(\R^d)\cap L^1(\R^d)\cap L^\infty(\R^d),
\qquad
\Psi_{\alpha,r}^\kappa>0,
\]
and
\[
\nabla_\alpha\Psi_{\alpha,r}^\kappa
\in W^{1,1}_{\mathrm{loc}}(\R^d;\R^{d_\alpha}).
\]
Moreover, $\nabla_\alpha\Psi_{\alpha,r}^\kappa$ and
$\nabla_\alpha^2\Psi_{\alpha,r}^\kappa$ are continuous.
Finally, for all $x\in\R^d$ and all $\alpha=1,\ldots,M$,
\begin{align}
 m_r^{\kappa,(\alpha)}(x)-x^{(\alpha)}u_r^\kappa(x)
 &=c_{1,\alpha}\nabla_\alpha\Psi_{\alpha,r}^\kappa(x),
 \label{eq:block-grad-id}\\
 (\lambda_{0,\alpha}I-\lambda_{1,\alpha}\Delta_\alpha)
 \Psi_{\alpha,r}^\kappa(x)
 &=c_{2,\alpha}u_r^\kappa(x),
 \label{eq:block-elliptic-id}\\
 a_r^{\kappa,(\alpha)}(x)-x^{(\alpha)}
 &=c_{1,\alpha}c_{2,\alpha}
 \frac{\nabla_\alpha\Psi_{\alpha,r}^\kappa(x)}
 {(\lambda_{0,\alpha}I-\lambda_{1,\alpha}\Delta_\alpha)
 \Psi_{\alpha,r}^\kappa(x)}.
 \label{eq:block-barycenter-id}
\end{align}
Here $m_r^{\kappa,(\alpha)}$ and $a_r^{\kappa,(\alpha)}$ denote the $\alpha$-th block components of $m_r^\kappa$ and $a_r^\kappa$, respectively.
\end{lemma}

\begin{proof}
The complete proof is given in \ref{app:block-proofs}.
\end{proof}

Applying the preceding Wronskian--renormalization argument to the blockwise companion identities now yields identifiability for the full product kernel.

\begin{theorem}[Field identifiability in the block-separable companion-elliptic class]\label{thm:block-separable-ce-identifiability}
For any $p,q\in\mathcal P(\R^d)$, the field $V_{p,q}^\kappa$ is well defined on all of $\R^d$, and
\[
V_{p,q}^\kappa\equiv0
\qquad\Longrightarrow\qquad
p=q.
\]
\end{theorem}

\begin{proof}
The complete proof is given in \ref{app:block-proofs}.
\end{proof}

Applying \Cref{cor:ce-classification} to each block factor shows that every kernel satisfying \Cref{ass:block-separable-ce} is a product of Gaussian or Mat\'ern factors across the blocks.
Up to positive scalar multiples, every such kernel can be written as
\[
\kappa(z)
=
\prod_{\alpha\in G}
\exp\left(-\frac{\norm{z^{(\alpha)}}^2}{2\sigma_\alpha^2}\right)
\prod_{\alpha\in H}
\left[
\ell_\alpha\norm{z^{(\alpha)}}^{\nu_\alpha}
K_{\nu_\alpha}\left(\frac{\norm{z^{(\alpha)}}}{\ell_\alpha}\right)
\right],
\]
where $G\cup H=\{1,\ldots,M\}$, $G\cap H=\emptyset$, and $\sigma_\alpha,\ell_\alpha>0$, $\nu_\alpha>0$.

The next example shows that this extension indeed contains anisotropic kernels not covered by the scalar companion-elliptic class.

\begin{example}[$\ell^1$-Laplace kernel]\label{ex:l1-laplace-block-ce}
Take $M=d$ and $d_1=\cdots=d_d=1$.
For each coordinate, choose the factor
\[
\kappa_j(t)=e^{-|t|/\tau_j},
\qquad \tau_j>0,
\]
which is a positive scalar multiple of the one-dimensional Mat\'ern kernel with smoothness $\nu=1/2$.
Hence
\[
\kappa(z)=\prod_{j=1}^d e^{-|z_j|/\tau_j}
=
\exp\left(-\sum_{j=1}^d\frac{|z_j|}{\tau_j}\right)
\]
satisfies \Cref{ass:block-separable-ce}; in particular, if $\tau_1=\cdots=\tau_d=\tau$, then $\kappa(z)=e^{-\norm{z}_1/\tau}$.
For $d\ge2$ this kernel is not Euclidean radial, and it therefore provides a concrete example distinct from the Gaussian and Euclidean Mat\'ern kernels arising in the scalar companion-elliptic classification.
\end{example}

\section{Stability}
\label{sec:quantitative-identifiability}
\label{sec:approximate}

\noindent This section answers (Q2) and (Q3); throughout, $\kappa$ is assumed to satisfy \Cref{ass:block-separable-ce}. \Cref{subsec:mass-escape-approx-wronskian} presents the mass-escape counterexample and extends the Wronskian relations of the exact case to nonvanishing fields, assembling the tools that the quantitative estimate requires. \Cref{subsec:explicit-first-stage-stability} proves the explicit stability inequality for the first identification stage, \Cref{subsec:deconvolution-two-stage} combines it with the amplification estimate for deconvolution and carries out the kernel comparison, and \Cref{subsec:qualitative-stability-tightness} establishes qualitative stability under tightness.

\subsection{Mass-escape counterexample and approximate Wronskian relations}
\label{subsec:mass-escape-approx-wronskian}

To make precise what it means for the field to be small only on bounded regions, we fix a metric for uniform convergence on compact sets: for continuous fields $F,G\in C(\R^d;\R^d)$, set
\begin{equation}\label{eq:compact-open-field-metric}
d_{\mathrm{loc}}(F,G)
:=\sum_{m=1}^\infty 2^{-m}
\min\left\{1,
\sup_{x\in\overline B_m}\|F(x)-G(x)\|
\right\}.
\end{equation}

\begin{counterexample}[Failure of global stability estimates due to mass escape]
\label{cex:local-approx-failure}
Let $p\in\mathcal P(\R^d)$ and $\eps\in(0,1)$, let $(z_n)\subset\R^d$ satisfy $\norm{z_n}\to\infty$, and set
\[
q_n=(1-\eps)p+\eps\delta_{z_n}.
\]
Then
\[
d_{\mathrm{loc}}(V_{p,q_n}^\kappa,0)\longrightarrow0,
\]
yet
\[
\|u_p^\kappa-u_{q_n}^\kappa\|_{L^1}
\longrightarrow2\eps\|\kappa\|_{L^1}>0.
\]
In particular, $q_n$ does not converge weakly to $p$.
\end{counterexample}

\begin{proof}
Fix a compact set $K$. For all $x$,
\[
u_{q_n}^\kappa(x)=(1-\eps)u_p^\kappa(x)+\eps\kappa(x-z_n),
\qquad
m_{q_n}^\kappa(x)=(1-\eps)m_p^\kappa(x)+\eps z_n\kappa(x-z_n),
\]
and subtracting with the aid of $a_p^\kappa=m_p^\kappa/u_p^\kappa$ gives
\[
a_{q_n}^\kappa(x)-a_p^\kappa(x)
=\frac{\eps\kappa(x-z_n)}
{(1-\eps)u_p^\kappa(x)+\eps\kappa(x-z_n)}
\bigl(z_n-a_p^\kappa(x)\bigr).
\]
By \Cref{lem:block-ce-identities},
\[
c_K:=\inf_K u_p^\kappa>0,
\qquad
A_K:=\sup_K\|a_p^\kappa(x)-x\|<\infty.
\]
Hence
\[
\sup_{x\in K}\|a_{q_n}^\kappa(x)-a_p^\kappa(x)\|
\le
\frac{\eps(1+A_K)}{(1-\eps)c_K}
\sup_{x\in K}(1+\norm{x-z_n})\kappa(x-z_n).
\]
From \Cref{ass:block-separable-ce,eq:block-kernel-identities} and the tensor-product closure of $C_0$ functions on product spaces \cite{folland1999real},
\[
\kappa\in C_0(\R^d),
\qquad
z\longmapsto z\kappa(z)\in C_0(\R^d;\R^d),
\]
so $h(z):=(1+\norm z)\kappa(z)\in C_0(\R^d)$. Moreover, $K-z_n$ escapes to infinity, so the last supremum converges to zero. Therefore $V_{p,q_n}^\kappa=a_p^\kappa-a_{q_n}^\kappa\to0$ uniformly on $K$.

On the other hand,
\[
u_{q_n}^\kappa=(1-\eps)u_p^\kappa+\eps\kappa(\,\cdot-z_n),
\qquad
\|u_p^\kappa-u_{q_n}^\kappa\|_{L^1}
=\eps\|u_p^\kappa-\kappa(\,\cdot-z_n)\|_{L^1}.
\]
The two functions $u_p^\kappa$ and $\kappa(\,\cdot-z_n)$ are nonnegative and both integrate to $\|\kappa\|_{L^1}$. Moreover, for any $R>0$,
\[
\int_{\R^d}\min\{u_p^\kappa(x),\kappa(x-z_n)\}\,dx
\le
|B_R|\sup_{x\in B_R}\kappa(x-z_n)
+\int_{B_R^c}u_p^\kappa(x)\,dx.
\]
First letting $n\to\infty$ sends the first term to zero, and then letting $R\to\infty$ sends the second term to zero as well. Hence
\[
\begin{aligned}
\|u_p^\kappa-\kappa(\,\cdot-z_n)\|_{L^1}
&=2\|\kappa\|_{L^1}
-2\int_{\R^d}\min\{u_p^\kappa(x),\kappa(x-z_n)\}\,dx\\
&\longrightarrow2\|\kappa\|_{L^1}.
\end{aligned}
\]
Finally, for every $R>0$ and all sufficiently large $n$, we have $z_n\notin B_R$, hence $q_n(B_R^c)\ge\eps$. Therefore $(q_n)$ is not tight and consequently cannot converge weakly to $p$ \cite{bogachev2007measure,kallenberg2021foundations}.
\end{proof}

\noindent This counterexample shows that there is no global modulus $\omega:[0,\infty)\to[0,\infty)$ satisfying
\begin{equation}\label{eq:no-global-first-stage-modulus}
\lim_{t\downarrow0}\omega(t)=0,
\qquad
\frac{\|u_p^\kappa-u_q^\kappa\|_{L^1}}
{\|\kappa\|_{L^1}}
\le
\omega\bigl(d_{\mathrm{loc}}(V_{p,q}^\kappa,0)\bigr)
\quad
\text{for every }q\in\mathcal P(\R^d).
\end{equation}
The cause is clear: $d_{\mathrm{loc}}$ observes the field only on bounded sets and is blind to mass escaping to infinity. A global estimate therefore requires a term measuring the mass outside the observation region, and when that mass cannot be measured, tightness must be invoked to rule out mass escape. The former leads to the tail term in \Cref{thm:explicit-first-stage-stability}, and the latter to \Cref{thm:tight-field-stability}.

\noindent Besides the tail term demanded by the counterexample, the quantification involves one more term originating in the structure of the exact proof. In the exact case, $V_{p,q}^\kappa\equiv0$ imposed the same first-order relation on the two companion potentials, and the Wronskian argument pinned their ratio to a constant. When the field is nonzero, the same computation leaves a field residual as a forcing term on the right-hand side. For Gaussian factors with $\lambda_{1,\alpha}=0$, the ratio of the potentials coincides exactly with the ratio of the convolution densities, whereas for Mat\'ern factors with $\lambda_{1,\alpha}>0$ the two ratios differ by the divergence of the Wronskian current, so the divergence of the weighted field residual must be controlled as well. For this quantification we establish an explicit $L^1$-Poincar\'e inequality on rectangular boxes and an $L^1$ estimate for transport equations with a forcing term.

We begin with the $L^1$-Poincar\'e inequality to be used on rectangular boxes.
For $m\ge1$ and bounded intervals $I_j=(a_j,b_j)$, let
$Q:=\prod_{j=1}^m I_j$ and
\begin{equation}\label{eq:box-poincare-constant}
 C_Q:=\frac12\left(\sum_{j=1}^m(b_j-a_j)^2\right)^{1/2}
 =\frac{\operatorname{diam}(Q)}{2}.
\end{equation}

\begin{lemma}[$L^1$-Poincar\'e inequality on a box]
\label{lem:box-l1-poincare}
For $g\in W^{1,1}(Q)$ and $g_Q:=|Q|^{-1}\int_Qg$,
\begin{equation}\label{eq:box-l1-poincare}
 \|g-g_Q\|_{L^1(Q)}
 \le C_Q\|\nabla g\|_{L^1(Q)}.
\end{equation}
\end{lemma}

\begin{proof}
First, on an interval $I=(a,b)$ with $\bar h_I:=|I|^{-1}\int_Ih$,
Fubini's theorem and the fundamental theorem of calculus give
\[
 \int_I|h-\bar h_I|
 \le \frac1{|I|}\int_I\int_I|h(x)-h(y)|\,dx\,dy
 \le \frac{|I|}{2}\int_I|h'(t)|\,dt.
\]
Applying the averaging operators in each coordinate successively, and using that
averaging is an $L^1$ contraction, we obtain
\[
 \|g-g_Q\|_{L^1(Q)}
 \le \sum_{j=1}^m\frac{b_j-a_j}{2}
 \|\partial_jg\|_{L^1(Q)}.
\]
Finally, integrating the pointwise Cauchy--Schwarz inequality in the index $j$,
\[
 \sum_{j=1}^m(b_j-a_j)\,|\partial_jg(x)|
 \le
 \left(\sum_{j=1}^m(b_j-a_j)^2\right)^{1/2}
 \norm{\nabla g(x)}
 =2C_Q\norm{\nabla g(x)},
\]
yields \eqref{eq:box-l1-poincare}.
\end{proof}

\noindent The constant $C_Q=\operatorname{diam}(Q)/2$ agrees with the optimal constant in the $L^1$-Poincar\'e inequality for convex domains \cite{acosta2004optimal}. The next lemma extends the mass conservation argument used in the exact case to equations with a forcing term.

\begin{lemma}[$L^1$ estimate for transport equations with forcing]
\label{lem:forced-transport-l1-contraction}
Let $\lambda>0$ and $b\in W^{1,1}_{\mathrm{loc}}(\R^d;\R^d)$.
Suppose $s\in L^1(\R^d)\cap L^\infty(\R^d)$ and $f\in L^1(\R^d)$ satisfy
\begin{equation}\label{eq:forced-transport-resolvent}
 s+\lambda\nabla\cdot(bs)=f
 \qquad\text{in }\mathcal D'(\R^d),
\end{equation}
and that $bs\in L^1(\R^d;\R^d)$.
Then
\begin{equation}\label{eq:forced-transport-l1-contraction}
 \|s\|_{L^1}\le \|f\|_{L^1}.
\end{equation}
\end{lemma}

\begin{proof}
For a standard spatial mollifier $\varphi_\varepsilon$, set
$s_\varepsilon:=s*\varphi_\varepsilon$,
$f_\varepsilon:=f*\varphi_\varepsilon$, and define
\[
 r_\varepsilon
 :=\nabla\cdot(bs_\varepsilon)
 -\bigl(\nabla\cdot(bs)\bigr)*\varphi_\varepsilon.
\]
Since $s\in L^\infty(\R^d)$ and
$b\in W^{1,1}_{\mathrm{loc}}(\R^d;\R^d)$,
the Friedrichs commutator estimate underlying \Cref{thm:renorm-sobolev}
\cite{diperna1989transport,ambrosio2004transport} gives
\begin{equation}\label{eq:forced-transport-commutator}
 r_\varepsilon\longrightarrow0
 \qquad\text{in }L^1_{\mathrm{loc}}(\R^d).
\end{equation}
Convolving \eqref{eq:forced-transport-resolvent} with $\varphi_\varepsilon$ yields
\begin{equation}\label{eq:forced-transport-mollified}
 s_\varepsilon+\lambda\nabla\cdot(bs_\varepsilon)
 =f_\varepsilon+\lambda r_\varepsilon
 \qquad\text{a.e. on }\R^d.
\end{equation}

Let $\beta\in C^1(\R)$ satisfy $\beta(0)=0$, $\beta'\in L^\infty$, and
$z\beta'(z)-\beta(z)\in L^\infty$.
Multiplying \eqref{eq:forced-transport-mollified} by $\beta'(s_\varepsilon)$ and
applying the chain rule gives
\begin{align}
 \beta'(s_\varepsilon)s_\varepsilon
 &+\lambda\nabla\cdot\bigl(b\beta(s_\varepsilon)\bigr)
 +\lambda\bigl(\beta'(s_\varepsilon)s_\varepsilon
              -\beta(s_\varepsilon)\bigr)\nabla\cdot b \notag\\
 &=\beta'(s_\varepsilon)f_\varepsilon
 +\lambda\beta'(s_\varepsilon)r_\varepsilon.
 \label{eq:forced-transport-renormalized-mollified}
\end{align}
The standard convergence properties of mollification together with \eqref{eq:forced-transport-commutator} allow passing to the limit
$\varepsilon\downarrow0$.
Indeed, the boundedness of $\beta'$ gives
$\beta'(s_\varepsilon)r_\varepsilon\to0$, and
the boundedness of $z\beta'(z)-\beta(z)$ together with dominated convergence handles
the term multiplied by $\nabla\cdot b$.
For the source term in particular,
\[
 \beta'(s_\varepsilon)f_\varepsilon-\beta'(s)f
 =\beta'(s_\varepsilon)(f_\varepsilon-f)
 +\bigl(\beta'(s_\varepsilon)-\beta'(s)\bigr)f
 \longrightarrow0
 \quad\text{in }L^1_{\mathrm{loc}},
\]
where the first term is handled by the boundedness of $\beta'$ and the second by
$s_\varepsilon\to s$ a.e.\ and dominated convergence.
Hence
\begin{align}
 \beta'(s)s
 &+\lambda\nabla\cdot\bigl(b\beta(s)\bigr)
 +\lambda\bigl(\beta'(s)s-\beta(s)\bigr)\nabla\cdot b
 =\beta'(s)f
 \label{eq:forced-transport-renormalized}
\end{align}
holds in the distributional sense.

Now take $\beta_\delta(z)=\sqrt{z^2+\delta^2}-\delta$.
Since $|\beta_\delta'|\le1$ and
$|z\beta_\delta'(z)-\beta_\delta(z)|\le\delta$,
letting $\delta\downarrow0$ in \eqref{eq:forced-transport-renormalized} gives
\begin{equation}\label{eq:forced-transport-kato}
 |s|+\lambda\nabla\cdot(b|s|)
 =\operatorname{sgn}(s)f
 \le |f|
 \qquad\text{in }\mathcal D'(\R^d).
\end{equation}
Choosing $\chi_R(x)=\chi(x/R)$ as in the proof of \Cref{lem:renorm-mass} and
testing \eqref{eq:forced-transport-kato} against $\chi_R$ yields
\[
 \int_{\R^d}\chi_R|s|\,dx
 \le
 \int_{\R^d}\chi_R|f|\,dx
 +\frac{\lambda\|\nabla\chi\|_{L^\infty}}{R}
  \|bs\|_{L^1}.
\]
Letting $R\to\infty$ and applying dominated convergence gives
\eqref{eq:forced-transport-l1-contraction}.
\end{proof}

We now extend the blockwise Wronskian computation to the case of a nonvanishing field. The four identities of the next lemma transcribe each step of the exact proof into a form carrying residuals. Fix $p,q\in\mathcal P(\R^d)$ and a block $\alpha$, and set
\[
 V_\alpha:=V_{p,q}^{\kappa,(\alpha)},
 \qquad
 L_\alpha:=\lambda_{0,\alpha}I-\lambda_{1,\alpha}\Delta_\alpha,
 \qquad
 b_{\alpha,q}:=
 \frac{a_q^{\kappa,(\alpha)}-x^{(\alpha)}}
 {c_{1,\alpha}c_{2,\alpha}}.
\]
Also set
\begin{align*}
 J_\alpha
 &:={\Psi}_{\alpha,q}^\kappa\nabla_\alpha{\Psi}_{\alpha,p}^\kappa
 -{\Psi}_{\alpha,p}^\kappa\nabla_\alpha{\Psi}_{\alpha,q}^\kappa,
 &s_\alpha&:=\nabla_\alpha\cdot J_\alpha,\\
 F_\alpha
 &:=\frac{{\Psi}_{\alpha,q}^\kappa u_p^\kappa}{c_{1,\alpha}}V_\alpha,
 &R_\alpha&:=\frac{{\Psi}_{\alpha,p}^\kappa}{{\Psi}_{\alpha,q}^\kappa},
 \qquad H:=\frac{u_p^\kappa}{u_q^\kappa}.
\end{align*}

\begin{lemma}[Approximate blockwise Wronskian relations]
\label{prop:approximate-block-wronskian}
With the notation above,
\begin{align}
 J_\alpha+\lambda_{1,\alpha}b_{\alpha,q}s_\alpha
 &=F_\alpha,
 \label{eq:approx-block-wronskian-current}\\
 s_\alpha+
 \lambda_{1,\alpha}\nabla_\alpha\cdot(b_{\alpha,q}s_\alpha)
 &=\nabla_\alpha\cdot F_\alpha,
 \label{eq:approx-block-wronskian-resolvent}\\
 J_\alpha
 &=({\Psi}_{\alpha,q}^\kappa)^2\nabla_\alpha R_\alpha,
 \label{eq:approx-block-ratio-gradient}\\
 H-R_\alpha
 &=-\frac{\lambda_{1,\alpha}}{c_{2,\alpha}}
 \frac{s_\alpha}{{\Psi}_{\alpha,q}^\kappa u_q^\kappa}.
 \label{eq:approx-block-ratio-error}
\end{align}
In particular, if $\lambda_{1,\alpha}=0$, then $H=R_\alpha$.
\end{lemma}

\begin{proof}
By \Cref{eq:block-barycenter-id,eq:block-elliptic-id},
\[
 \nabla_\alpha\Psi_{\alpha,r}^\kappa
 =b_{\alpha,r}L_\alpha\Psi_{\alpha,r}^\kappa,
 \qquad
 b_{\alpha,p}-b_{\alpha,q}
 =\frac{V_\alpha}{c_{1,\alpha}c_{2,\alpha}},
 \qquad
 L_\alpha\Psi_{\alpha,r}^\kappa=c_{2,\alpha}u_r^\kappa.
\]
Hence
\begin{align*}
 J_\alpha
 &=b_{\alpha,q}
 \left(\Psi_{\alpha,q}^\kappa L_\alpha\Psi_{\alpha,p}^\kappa
 -\Psi_{\alpha,p}^\kappa L_\alpha\Psi_{\alpha,q}^\kappa\right)
 +\Psi_{\alpha,q}^\kappa
 (b_{\alpha,p}-b_{\alpha,q})L_\alpha\Psi_{\alpha,p}^\kappa\\
 &=-\lambda_{1,\alpha}b_{\alpha,q}s_\alpha+F_\alpha,
\end{align*}
which proves \eqref{eq:approx-block-wronskian-current}.
Taking the divergence gives \eqref{eq:approx-block-wronskian-resolvent}, and
the quotient rule gives \eqref{eq:approx-block-ratio-gradient}.
Finally, since
\[
 L_\alpha(R_\alpha\Psi_{\alpha,q}^\kappa)
 -R_\alpha L_\alpha\Psi_{\alpha,q}^\kappa
 =-\lambda_{1,\alpha}
 \frac{s_\alpha}{\Psi_{\alpha,q}^\kappa},
\]
substituting $L_\alpha\Psi_{\alpha,r}^\kappa=c_{2,\alpha}u_r^\kappa$ and
dividing by $u_q^\kappa$ yields
\eqref{eq:approx-block-ratio-error}.
\end{proof}

\subsection{Explicit stability of the first identification stage}
\label{subsec:explicit-first-stage-stability}
With the tools of \Cref{subsec:mass-escape-approx-wronskian} in place, we assemble the quantities entering the bound. Fix $p,q\in\mathcal P(\R^d)$ and a bounded rectangular box
\[
 \Omega_\alpha=\prod_{j=1}^{d_\alpha}(a_{\alpha j},b_{\alpha j}),
 \qquad
 \Omega:=\prod_{\alpha=1}^M\Omega_\alpha,
\]
and write $C_{\Omega_\alpha}$ for the constant in \eqref{eq:box-poincare-constant}. Define
\begin{align*}
 K_\kappa&:=\|\kappa\|_{L^1},
 &U_{r,\Omega}&:=\int_\Omega u_r^\kappa(x)\,dx\quad(r=p,q),\\
 T_{r,\Omega}&:=\int_{\Omega^c}u_r^\kappa(x)\,dx\quad(r=p,q),
 &\underline u_\Omega&:=\inf_{\overline\Omega}u_q^\kappa,
 \qquad
 \overline u_\Omega:=\sup_{\overline\Omega}u_q^\kappa.
\end{align*}
For each block $\alpha$, set
\begin{align*}
 \underline\psi_{\alpha,\Omega}
 &:=\inf_{\overline\Omega}\Psi_{\alpha,q}^\kappa,
 &B_{\alpha,\Omega}
 &:=\sup_{\overline\Omega}\|b_{\alpha,q}\|,\\
 E_{\alpha,\Omega}^{(0)}
 &:=\|F_\alpha\|_{L^1(\Omega)},
 &E_\alpha^{(1)}
 &:=
 \begin{cases}
 0,&\lambda_{1,\alpha}=0,\\
 \|\nabla_\alpha\cdot F_\alpha\|_{L^1(\R^d)},
 &\lambda_{1,\alpha}>0,
 \end{cases}
\end{align*}
and define
\begin{equation}\label{eq:block-first-stage-local-certificate}
 \mathcal E_{\alpha,\Omega}(p,q)
 :=
 C_{\Omega_\alpha}\underline\psi_{\alpha,\Omega}^{-2}
 \left(
 E_{\alpha,\Omega}^{(0)}
 +\lambda_{1,\alpha}B_{\alpha,\Omega}E_\alpha^{(1)}
 \right)
 +\frac{2\lambda_{1,\alpha}}
 {c_{2,\alpha}\underline\psi_{\alpha,\Omega}\underline u_\Omega}
 E_\alpha^{(1)}.
\end{equation}
By the companion identities, all denominators above are positive. Here $T_{r,\Omega}$ measures the convolution mass outside the box, $E_{\alpha,\Omega}^{(0)}$ and $E_\alpha^{(1)}$ are the field and divergence residuals, and $\mathcal E_{\alpha,\Omega}$ combines them into the blockwise certificate.

\begin{theorem}[Explicit stability of the first identification stage]
\label{thm:explicit-first-stage-stability}
Suppose $\kappa$ satisfies \Cref{ass:block-separable-ce}. If, for each block with $\lambda_{1,\alpha}>0$,
\[
 F_\alpha\in L^1(\R^d;\R^{d_\alpha}),
 \qquad
 \nabla_\alpha\cdot F_\alpha\in L^1(\R^d),
\]
then
\begin{equation}\label{eq:explicit-first-stage-stability}
 {
 \|u_p^\kappa-u_q^\kappa\|_{L^1(\R^d)}
 \le
 2\overline u_\Omega
 \sum_{\alpha=1}^M\mathcal E_{\alpha,\Omega}(p,q)
 +2\max\{T_{p,\Omega},T_{q,\Omega}\}.
 }
\end{equation}
\end{theorem}

\begin{proof}
We follow the notation of \Cref{prop:approximate-block-wronskian}.
If $\lambda_{1,\alpha}>0$, then
\[
 s_\alpha
 =\frac{c_{2,\alpha}}{\lambda_{1,\alpha}}
 \left(\Psi_{\alpha,p}^\kappa u_q^\kappa
 -\Psi_{\alpha,q}^\kappa u_p^\kappa\right)
 \in L^1\cap L^\infty.
\]
Moreover $J_\alpha\in L^1$, and
\eqref{eq:approx-block-wronskian-current} together with $F_\alpha\in L^1$ gives
$b_{\alpha,q}s_\alpha\in L^1$. Furthermore,
\Cref{lem:block-ce-identities} and $u_q^\kappa>0$ give
\[
b_{\alpha,q}
=\frac{\nabla_\alpha\Psi_{\alpha,q}^\kappa}
{c_{2,\alpha}u_q^\kappa}
\in W^{1,1}_{\mathrm{loc}}(\R^d;\R^{d_\alpha}).
\]
Regarding the $\R^{d_\alpha}$-valued field $b_{\alpha,q}$ as an
$\R^d$-valued field vanishing on the remaining blocks and applying
\Cref{lem:forced-transport-l1-contraction} to
\eqref{eq:approx-block-wronskian-resolvent} gives
\begin{equation}\label{eq:s-alpha-contraction}
 \|s_\alpha\|_{L^1}\le E_\alpha^{(1)}.
\end{equation}
When $\lambda_{1,\alpha}=0$, all $s_\alpha$ terms below vanish.
Hence \eqref{eq:approx-block-wronskian-current} yields
\begin{equation}\label{eq:J-alpha-local-bound}
 \|J_\alpha\|_{L^1(\Omega)}
 \le E_{\alpha,\Omega}^{(0)}
 +\lambda_{1,\alpha}B_{\alpha,\Omega}E_\alpha^{(1)}.
\end{equation}

Set $H=u_p^\kappa/u_q^\kappa$ and
$R_\alpha=\Psi_{\alpha,p}^\kappa/\Psi_{\alpha,q}^\kappa$.
By \Cref{eq:approx-block-ratio-error,eq:s-alpha-contraction},
\begin{equation}\label{eq:H-R-alpha-bound}
 \|H-R_\alpha\|_{L^1(\Omega)}
 \le
 \frac{\lambda_{1,\alpha}}
 {c_{2,\alpha}\underline\psi_{\alpha,\Omega}\underline u_\Omega}
 E_\alpha^{(1)}.
\end{equation}
Let $P_\alpha$ denote the averaging operator with respect to $x^{(\alpha)}$ over $\Omega_\alpha$.
From \Cref{lem:box-l1-poincare,eq:approx-block-ratio-gradient,eq:J-alpha-local-bound},
\begin{equation}\label{eq:R-alpha-fiber-bound}
 \|R_\alpha-P_\alpha R_\alpha\|_{L^1(\Omega)}
 \le
 C_{\Omega_\alpha}\underline\psi_{\alpha,\Omega}^{-2}
 \left(
 E_{\alpha,\Omega}^{(0)}
 +\lambda_{1,\alpha}B_{\alpha,\Omega}E_\alpha^{(1)}
 \right).
\end{equation}
Moreover, since $P_\alpha$ is an $L^1$ contraction, \eqref{eq:H-R-alpha-bound} gives
\[
 \|(I-P_\alpha)(H-R_\alpha)\|_{L^1(\Omega)}
 \le2\|H-R_\alpha\|_{L^1(\Omega)}
 \le\frac{2\lambda_{1,\alpha}}
 {c_{2,\alpha}\underline\psi_{\alpha,\Omega}\underline u_\Omega}
 E_\alpha^{(1)}.
\]
Since $(I-P_\alpha)H=(I-P_\alpha)R_\alpha+(I-P_\alpha)(H-R_\alpha)$, adding this to \eqref{eq:R-alpha-fiber-bound} gives
\begin{equation}\label{eq:H-coordinate-independent-approx}
 \|(I-P_\alpha)H\|_{L^1(\Omega)}
 \le\mathcal E_{\alpha,\Omega}(p,q).
\end{equation}

The averaging operators commute with one another and are $L^1$ contractions.
Setting $P:=P_1\cdots P_M$, the function $PH$ is the constant
$\overline H_\Omega:=|\Omega|^{-1}\int_\Omega H$ on $\Omega$, so a telescoping decomposition together with \eqref{eq:H-coordinate-independent-approx} gives
\begin{align}
 \|H-\overline H_\Omega\|_{L^1(\Omega)}
 &\le\sum_{\alpha=1}^M\|(I-P_\alpha)H\|_{L^1(\Omega)}
 \le\sum_{\alpha=1}^M\mathcal E_{\alpha,\Omega}(p,q).
 \label{eq:H-global-constant-bound}
\end{align}
That is, the approximate constancy obtained in each block combines into approximate constancy on the whole box.

We now pin down the local multiplicative constant by the probability normalization.
Set
\[
 c_\Omega:=
 \frac{\int_\Omega u_p^\kappa}{\int_\Omega u_q^\kappa}
 =\frac{K_\kappa-T_{p,\Omega}}
 {K_\kappa-T_{q,\Omega}}.
\]
Then
\begin{align*}
 \|u_p^\kappa-c_\Omega u_q^\kappa\|_{L^1(\Omega)}
 &=\int_\Omega u_q^\kappa|H-c_\Omega|\\
 &\le
 \overline u_\Omega\|H-\overline H_\Omega\|_{L^1(\Omega)}
 +U_{q,\Omega}|\overline H_\Omega-c_\Omega|.
\end{align*}
On the other hand,
\[
 U_{q,\Omega}|\overline H_\Omega-c_\Omega|
 =\left|\int_\Omega
 (\overline H_\Omega-H)u_q^\kappa\right|
 \le\overline u_\Omega
 \|H-\overline H_\Omega\|_{L^1(\Omega)}.
\]
Hence \eqref{eq:H-global-constant-bound} gives
\begin{equation}\label{eq:local-scaled-convolution-bound}
 \|u_p^\kappa-c_\Omega u_q^\kappa\|_{L^1(\Omega)}
 \le2\overline u_\Omega
 \sum_{\alpha=1}^M\mathcal E_{\alpha,\Omega}(p,q).
\end{equation}
It remains to replace $c_\Omega$ by $1$ and to handle the exterior of the box. Since $U_{r,\Omega}=K_\kappa-T_{r,\Omega}$,
\[
 |1-c_\Omega|\,U_{q,\Omega}
 =|U_{p,\Omega}-U_{q,\Omega}|
 =|T_{p,\Omega}-T_{q,\Omega}|,
\]
and combining this with \eqref{eq:local-scaled-convolution-bound} gives
\[
\begin{aligned}
 \|u_p^\kappa-u_q^\kappa\|_{L^1(\Omega)}
 &\le\|u_p^\kappa-c_\Omega u_q^\kappa\|_{L^1(\Omega)}
 +|1-c_\Omega|\,U_{q,\Omega}\\
 &\le2\overline u_\Omega
 \sum_{\alpha=1}^M\mathcal E_{\alpha,\Omega}(p,q)
 +|T_{p,\Omega}-T_{q,\Omega}|.
\end{aligned}
\]
Outside the box, $u_p^\kappa,u_q^\kappa\ge0$ gives
\[
 \|u_p^\kappa-u_q^\kappa\|_{L^1(\Omega^c)}
 \le T_{p,\Omega}+T_{q,\Omega}.
\]
Adding the two inequalities and using
\[
 |T_{p,\Omega}-T_{q,\Omega}|
 +T_{p,\Omega}+T_{q,\Omega}
 =2\max\{T_{p,\Omega},T_{q,\Omega}\}
\]
yields \eqref{eq:explicit-first-stage-stability}.
\end{proof}

\begin{remark}
\label{rem:first-stage-assumptions}
If $V_{p,q}^\kappa\equiv0$, then all residual terms vanish, and if $\Omega$ is expanded to $\R^d$ so that the tail terms tend to $0$, then inequality \eqref{eq:explicit-first-stage-stability} recovers $u_p^\kappa=u_q^\kappa$, that is, the first identification stage of the exact case. Furthermore, the coefficient $2$ of the tail term cannot be improved. Suppose all factors are Gaussian and apply the theorem to the mass-escape counterexample $q_n=(1-\eps)p+\eps\delta_{z_n}$. Since the Gaussian residuals involve only quantities on $\Omega$, for fixed $\Omega$ we have $\mathcal E_{\alpha,\Omega}(p,q_n)\to0$ and $T_{q_n,\Omega}\to(1-\eps)T_{p,\Omega}+\eps K_\kappa$ as $n\to\infty$. The right-hand side therefore converges to $2\bigl((1-\eps)T_{p,\Omega}+\eps K_\kappa\bigr)$ and the left-hand side to $2\eps K_\kappa$, and expanding $\Omega$ so that $T_{p,\Omega}\downarrow0$ makes the two limits coincide.
\end{remark}

To obtain a sharper bound for the Gaussian case, let
\[
 \kappa(z)=\prod_{\alpha=1}^M
 \exp\left(-\frac{\|z^{(\alpha)}\|^2}{2\sigma_\alpha^2}\right).
\]
For a bounded rectangular box $\Omega$, write $C_\Omega$ for the constant in \eqref{eq:box-poincare-constant} and set
\[
 W_{p,q}(x):=
 \left(
 \sum_{\alpha=1}^M
 \frac{\|V_{p,q}^{\kappa,(\alpha)}(x)\|^2}{\sigma_\alpha^4}
 \right)^{1/2}.
\]

\begin{corollary}[Explicit bound for Gaussian product kernels]
\label{cor:gaussian-explicit-first-stage}
With the notation of \Cref{thm:explicit-first-stage-stability},
\begin{equation}\label{eq:gaussian-explicit-first-stage}
 {
 \|u_p^\kappa-u_q^\kappa\|_{L^1}
 \le
 2C_\Omega\frac{\overline u_\Omega}{\underline u_\Omega}
 \int_\Omega u_p^\kappa(x)W_{p,q}(x)\,dx
 +2\max\{T_{p,\Omega},T_{q,\Omega}\}.
 }
\end{equation}
In particular, in the isotropic case $\sigma_\alpha=\sigma$, the first term simplifies to
\[
 \frac{2C_\Omega}{\sigma^2}
 \frac{\overline u_\Omega}{\underline u_\Omega}
 \int_\Omega u_p^\kappa\|V_{p,q}^\kappa\|.
\]
\end{corollary}

\begin{proof}
By the Gaussian score identity, $H=u_p^\kappa/u_q^\kappa$ satisfies
\[
 \nabla_\alpha H
 =H\frac{V_{p,q}^{\kappa,(\alpha)}}{\sigma_\alpha^2}.
\]
Setting $c_\Omega=U_{p,\Omega}/U_{q,\Omega}$ and
$\overline H_\Omega:=|\Omega|^{-1}\int_\Omega H$ and using
\Cref{lem:box-l1-poincare},
\begin{align*}
 \int_\Omega u_q^\kappa|H-c_\Omega|
 &\le2\overline u_\Omega
 \|H-\overline H_\Omega\|_{L^1(\Omega)}\\
 &\le2C_\Omega\overline u_\Omega
 \|\nabla H\|_{L^1(\Omega)}\\
 &\le2C_\Omega\frac{\overline u_\Omega}{\underline u_\Omega}
 \int_\Omega u_p^\kappa W_{p,q}.
\end{align*}
The subsequent normalization of the constant and the tail estimate are as in the proof of
\Cref{thm:explicit-first-stage-stability}.
\end{proof}

\subsection{Combining the two identification stages and kernel selection}
\label{subsec:deconvolution-two-stage}

\noindent Having obtained the explicit bound for the first stage, we now examine how much the error is amplified across frequencies in the Fourier deconvolution. First, since multiplying the kernel by a positive constant changes neither $a_r^\kappa$ nor $V_{p,q}^\kappa$, we normalize
\[
\widehat\kappa(0)=\int_{\R^d}\kappa(x)\,dx=1.
\]

For $T>0$, set
\begin{equation}\label{eq:spectral-amplification}
A_\kappa(T):=
\left(\min_{\|\xi\|\le T}\widehat\kappa(\xi)\right)^{-1}.
\end{equation}
If $f$ is a bounded integrable continuous function with $\widehat f\in L^1(\R^d)$ and $\operatorname{supp}\widehat f\subseteq\{\|\xi\|\le T\}$, then
\begin{equation}\label{eq:bandlimited-deconvolution}
\left|\int_{\R^d}f\,dp-\int_{\R^d}f\,dq\right|
\le
(2\pi)^{-d}\|\widehat f\|_{L^1}
A_\kappa(T)\|u_p^\kappa-u_q^\kappa\|_{L^1}.
\end{equation}
Indeed, $\widehat{u_p^\kappa-u_q^\kappa}=\widehat\kappa(\widehat p-\widehat q)$, so it suffices to apply Fourier inversion.

By \Cref{lem:ce-spectral-law}, for the normalized Gaussian and Mat\'ern kernels,
\[
A_{\kappa_\sigma}(T)
=\exp\!\left(\frac{\sigma^2T^2}{2}\right),
\qquad
A_{\kappa_{\nu,\ell}}(T)
=(1+\ell^2T^2)^{\nu+d/2}.
\]
In particular, matching the low-frequency quadratic decay rates of the Gaussian and the Euclidean Laplace kernel via $\sigma^2=(d+1)\ell^2$, we have, for all $T>0$,
\begin{equation}\label{eq:laplace-gaussian-amplification}
A_{\kappa_{1/2,\ell}}(T)
=(1+\ell^2T^2)^{(d+1)/2}
<
\exp\!\left(\frac{d+1}{2}\ell^2T^2\right)
=A_{\kappa_\sigma}(T).
\end{equation}
Hence, for the same value of $\|u_p^\kappa-u_q^\kappa\|_{L^1}$, the Laplace kernel amplifies the error less in the second stage than the Gaussian kernel does.

Combining the first-stage bound with the deconvolution estimate now shows how field errors propagate to differences in expectations of band-limited test functions.

\begin{corollary}[Combination of the two identification stages]
\label{cor:explicit-two-stage-identifiability}
Assume the hypotheses of \Cref{thm:explicit-first-stage-stability}. Normalize $\widehat\kappa(0)=\|\kappa\|_{L^1}=1$ and let $A_\kappa(T)$ be as in \eqref{eq:spectral-amplification}.
If $f$ is a bounded integrable continuous function with
$\widehat f\in L^1(\R^d)$ and
$\operatorname{supp}\widehat f\subseteq\{\|\xi\|\le T\}$, then
\begin{align}
 \left|\int f\,dp-\int f\,dq\right|
 &\le(2\pi)^{-d}\|\widehat f\|_{L^1}A_\kappa(T)
 \Bigg[
 2\overline u_\Omega
 \sum_{\alpha=1}^M\mathcal E_{\alpha,\Omega}(p,q)
 +2\max\{T_{p,\Omega},T_{q,\Omega}\}
 \Bigg].
 \label{eq:explicit-two-stage-identifiability}
\end{align}
\end{corollary}

\begin{proof}
Substitute \eqref{eq:explicit-first-stage-stability} into \eqref{eq:bandlimited-deconvolution}.
\end{proof}

\begin{remark}[Two-stage stability and kernel comparison]
\label{rem:two-stage-kernel-conditioning}
For a fixed pair of distributions $(p,q)$ and a rectangular box $\Omega$, let $\mathfrak B_\kappa(p,q;\Omega)$ denote, for each candidate kernel, the smallest applicable first-stage bound. Then the kernel-dependent part of \Cref{cor:explicit-two-stage-identifiability} is
\[
A_\kappa(T)\,\mathfrak B_\kappa(p,q;\Omega).
\]
Here $\mathfrak B_\kappa$ represents the error incurred in recovering the difference of the convolution densities from the field, and $A_\kappa(T)$ the amplification incurred in deconvolving the distributional information up to frequency $T$ from that difference. A full comparison must therefore consider the two factors together.

To compare the Gaussian and a scalar Mat\'ern kernel on an equal footing, we match their low-frequency quadratic decay rates by setting
\[
\beta:=\nu+\frac d2,
\qquad
\sigma^2=2\beta\ell^2,
\qquad
x:=\ell^2T^2.
\]
With $\mathfrak B_G$ and $\mathfrak B_M$ denoting the first-stage bounds of the two kernels, the kernel-dependent parts of the overall bounds are
\[
\mathfrak C_G(T)=\mathfrak B_G e^{\beta x},
\qquad
\mathfrak C_M(T)=\mathfrak B_M(1+x)^\beta.
\]
The Mat\'ern overall bound is therefore smaller than the Gaussian one if and only if
\begin{equation}\label{eq:concise-gaussian-matern-selection}
\frac{\mathfrak B_M}{\mathfrak B_G}
\le
\left(\frac{e^x}{1+x}\right)^\beta.
\end{equation}
Here one may use the sharper bound of \Cref{cor:gaussian-explicit-first-stage} for $\mathfrak B_G$. If $\mathfrak B_M\le\mathfrak B_G$, the Mat\'ern bound is smaller for every $T>0$. If instead $\mathfrak B_M>\mathfrak B_G$, there is a unique $x_\ast>0$ satisfying
\[
\beta\bigl[x_\ast-\log(1+x_\ast)\bigr]
=
\log\frac{\mathfrak B_M}{\mathfrak B_G}.
\]
For frequency bands below $T_\ast:=\ell^{-1}\sqrt{x_\ast}$ the Gaussian bound is smaller, and above $T_\ast$ the Mat\'ern bound is smaller.

This comparison shows that the Mat\'ern kernel is the more advantageous choice. Squeezing the integrand of $x-\log(1+x)=\int_0^x t(1+t)^{-1}\,dt$ between $t(1+x)^{-1}$ and $t$ gives
\[
\frac{x^2}{2(1+x)}\le x-\log(1+x)\le\frac{x^2}{2}
\qquad(x\ge0),
\]
so with $\Lambda:=\beta^{-1}\log(\mathfrak B_M/\mathfrak B_G)$ we obtain
\begin{equation}\label{eq:crossover-localization}
\sqrt{2\Lambda}\;\le\;x_\ast\;\le\;\sqrt{2\Lambda}+2\Lambda,
\end{equation}
and in particular $(2\Lambda)^{1/4}\le\ell T_\ast\le\sqrt{2}\,(2\Lambda)^{1/4}$ whenever $\Lambda\le1/2$. On the other hand, the right-hand side of \eqref{eq:concise-gaussian-matern-selection} already equals $(e/2)^\beta$ at the kernel's own frequency scale $T=\ell^{-1}$, that is, at $x=1$. For the Gaussian bound to remain smaller up to this scale one would need $\mathfrak B_M/\mathfrak B_G\ge(e/2)^\beta$, and since $\beta=\nu+d/2$, this threshold grows exponentially in the dimension. As long as the ratio of the first-stage bounds stays below this threshold, $\Lambda$ is small and the crossover frequency $T_\ast$ lies far below the kernel scale $\ell^{-1}$. The Mat\'ern bound is therefore smaller on all bands except this narrow low-frequency range. The cases in which the Gaussian retains a practical advantage are confined to those where the divergence residual cannot be controlled or where $\mathfrak B_M/\mathfrak B_G$ exceeds a threshold exponential in the dimension.
\end{remark}

\subsection{Qualitative stability}
\label{subsec:qualitative-stability-tightness}

\noindent Applying the preceding explicit bounds requires knowing, in addition to the field residual, the divergence residual of the Mat\'ern factors and the convolution mass outside the observation region. But what if none of these quantities are known and one is given only that the field converges to zero at each point? The next theorem shows that adding tightness alone to this weak information yields weak convergence together with $L^1$ convergence of the convolution densities.

\begin{theorem}[Qualitative stability under tightness]
\label{thm:tight-field-stability}
Fix $p\in\mathcal P(\R^d)$. Suppose $(q_n)\subset\mathcal P(\R^d)$ is tight and, for every $x\in\R^d$,
\[
V_{p,q_n}^\kappa(x)\longrightarrow0.
\]
Then
\[
q_n\Rightarrow p,
\qquad
\|u_{q_n}^\kappa-u_p^\kappa\|_{L^1}\longrightarrow0.
\]
\end{theorem}

\begin{proof}
Let $(q_{n_k})$ be an arbitrary subsequence. By tightness and Prokhorov's theorem \cite{bogachev2007measure,kallenberg2021foundations}, some further subsequence, still denoted $q_{n_k}$, converges weakly to a probability measure $q$.

For fixed $x$ and coordinate $j$, it follows from \Cref{ass:block-separable-ce,eq:block-kernel-identities} and the standard closure properties of $C_0$ under tensor products, translations, reflections, and linear combinations \cite{folland1999real} that
\[
y\longmapsto\kappa(x-y),
\qquad
y\longmapsto y_j\kappa(x-y)
\]
belong to $C_0(\R^d)\subset C_b(\R^d)$. By the characterization of weak convergence through bounded continuous test functions \cite{bogachev2007measure,kallenberg2021foundations},
\[
u_{q_{n_k}}^\kappa(x)\longrightarrow u_q^\kappa(x),
\qquad
m_{q_{n_k},j}^\kappa(x)\longrightarrow m_{q,j}^\kappa(x).
\]
Since $u_q^\kappa(x)>0$, we get $a_{q_{n_k}}^\kappa(x)\to a_q^\kappa(x)$. By the field convergence, this limit equals $a_p^\kappa(x)$, so $V_{p,q}^\kappa\equiv0$, and \Cref{thm:block-separable-ce-identifiability} gives $q=p$.

Thus every subsequence admits a further subsequence converging weakly to $p$. Since the weak topology on $\mathcal P(\R^d)$ is metrizable \cite{bogachev2007measure,kallenberg2021foundations}, the full sequence satisfies $q_n\Rightarrow p$.

In particular, applying weak convergence once more at each $x\in\R^d$ gives $u_{q_n}^\kappa(x)\to u_p^\kappa(x)$. Moreover, for all $n$,
\[
\int_{\R^d}u_{q_n}^\kappa(x)\,dx
=
\int_{\R^d}u_p^\kappa(x)\,dx
=
\|\kappa\|_{L^1}.
\]
Since both functions are nonnegative, applying dominated convergence to $\min\{u_{q_n}^\kappa,u_p^\kappa\}\le u_p^\kappa$ gives
\[
\begin{aligned}
\|u_{q_n}^\kappa-u_p^\kappa\|_{L^1}
&=2\|\kappa\|_{L^1}
-2\int_{\R^d}\min\{u_{q_n}^\kappa,u_p^\kappa\}\,dx\\
&\longrightarrow0.
\end{aligned}
\]
\end{proof}

\begin{remark}[Tightness in the training process]
\label{rem:tightness-training-condition}
Tightness is a natural assumption for the training process: if the data and the model outputs lie in a bounded region, or if a projection keeps the outputs within a bounded set, then the sequence of model distributions is tight.
\end{remark}

\section{Conclusion}\label{sec:discussion}

This paper has answered the three questions of the introduction for the drifting field with its original normalization and displacement. The analysis rests on a single observation: for a kernel admitting a companion potential, a vanishing field forces the Wronskian current of the two potentials into a stationary continuity equation, and identifiability follows by renormalization and Fourier deconvolution. Throughout the companion-elliptic class and its block-separable extension, the zero-field equilibrium is therefore equivalent to agreement with the data distribution for arbitrary Borel probability measures. This structure is rigid: up to positive scalar multiples, the scalar kernels carrying it are exactly the Gaussian and Mat\'ern families. Deconvolution alone is less demanding: it succeeds precisely when the zero set of $\widehat\kappa$ has empty interior, and beyond this boundary the field cannot distinguish certain pairs of distinct smooth positive densities.

At approximate equilibria the obstruction is mass escape: however small the field is on a bounded region, distant mass leaves the global error at order one. The explicit stability inequality names the missing quantity, the convolution mass outside the observation region; when this mass cannot be measured, tightness excludes escape and turns pointwise field convergence into weak convergence. The same bounds guide the choice of kernel. The Gaussian requires only field values but amplifies high-frequency errors exponentially, whereas the Mat\'ern requires in addition the divergence of the weighted field residual in exchange for polynomial amplification. Since this gap widens exponentially with frequency, the comparison favors the Mat\'ern family and supports the Laplace default of Deng et al.~\cite{deng2026drifting}.

Three questions remain open: whether the original field identifies distributions for general kernels with everywhere positive $\widehat\kappa$, as deconvolution alone permits; whether combining the present population-level bounds with the statistical error of an estimated field yields finite-sample guarantees; and whether the training dynamics in fact drive the field to zero, and at what rate.

\appendix
\section{Complete proofs for the block-separable class}
\label{app:block-proofs}

\noindent This appendix presents the complete proofs of \Cref{lem:block-ce-identities,thm:block-separable-ce-identifiability}. The notation and assumptions follow \Cref{sec:preliminaries,subsec:block-separable-ce-identifiability}.

\subsection{Proof of the companion identities}

The proof lifts the factor properties to the product kernel through Fubini's theorem; the one delicate point is the local Sobolev regularity of $\nabla_\alpha\Psi_{\alpha,r}^\kappa$, which requires bounds uniform over translates.

\begin{proof}[Proof of \Cref{lem:block-ce-identities}]
We first lift the factor-level regularity to the full space.
For each $\alpha$, the factor $\kappa_\alpha$ satisfies \Cref{ass:companion-elliptic} on $\R^{d_\alpha}$, so
\[
\kappa_\alpha\in C_0(\R^{d_\alpha})\cap L^1(\R^{d_\alpha})
\cap L^\infty(\R^{d_\alpha})\cap W^{1,1}(\R^{d_\alpha}).
\]
Hence $\kappa(z)=\prod_{\alpha=1}^M\kappa_\alpha(z^{(\alpha)})$ is continuous and bounded.
As $\norm z\to\infty$, at least one block norm tends to infinity, and the corresponding factor converges to zero while the remaining factors stay bounded.
Therefore $\kappa\in C_0(\R^d)$.
By Fubini's theorem,
\[
\|\kappa\|_{L^1(\R^d)}
=
\prod_{\alpha=1}^M
\|\kappa_\alpha\|_{L^1(\R^{d_\alpha})},
\qquad
\|\kappa\|_{L^\infty(\R^d)}
\le
\prod_{\alpha=1}^M
\|\kappa_\alpha\|_{L^\infty(\R^{d_\alpha})}.
\]
If the coordinate $x_j$ belongs to the $\alpha$-th block, the distributional derivative is
\[
\partial_j\kappa(z)
=
(\partial_j\kappa_\alpha)(z^{(\alpha)})
\prod_{\beta\ne\alpha}\kappa_\beta(z^{(\beta)}),
\]
and applying Fubini's theorem again yields
\[
\|\partial_j\kappa\|_{L^1(\R^d)}
=
\|\partial_j\kappa_\alpha\|_{L^1(\R^{d_\alpha})}
\prod_{\beta\ne\alpha}
\|\kappa_\beta\|_{L^1(\R^{d_\beta})}.
\]
Hence $\kappa\in W^{1,1}(\R^d)$.

We next verify the finiteness of the first moment.
The first identity in \eqref{eq:block-kernel-identities} gives
\[
z^{(\alpha)}\kappa(z)
=-c_{1,\alpha}\nabla_\alpha\eta^{[\alpha]}(z).
\]
Moreover,
\[
\nabla_\alpha\eta^{[\alpha]}(z)
=
\nabla_\alpha\eta_\alpha(z^{(\alpha)})
\prod_{\beta\ne\alpha}\kappa_\beta(z^{(\beta)}),
\]
and the right-hand side belongs to $C_0(\R^d)\cap L^1(\R^d)\cap L^\infty(\R^d)$.
Hence $z\mapsto z\kappa(z)$ is bounded and continuous and vanishes at infinity.
Consequently, for every $x\in\R^d$,
\[
\int_{\R^d}\norm{y-x}\kappa(x-y)\,r(dy)
=
\int_{\R^d}\norm z\,\kappa(z)\,(\tau_x)_\# r(dz)
\le
\|\norm{\cdot}\,\kappa\|_{L^\infty}<\infty,
\]
where $\tau_x(y)=x-y$.
Therefore $D_r^\kappa=\R^d$, and $m_r^\kappa$ and $a_r^\kappa$ are defined at every point.

We now establish the required convolution regularity.
Applying (F2) of \Cref{sec:preliminaries} to $\kappa$ and the probability measure $r$ gives
\[
u_r^\kappa=\kappa*r\in C_0(\R^d)\cap L^1(\R^d),
\qquad
\|u_r^\kappa\|_{L^1}\le \|\kappa\|_{L^1}.
\]
Moreover,
\[
|u_r^\kappa(x)|
\le
\int \|\kappa\|_{L^\infty}\,r(dy)
=
\|\kappa\|_{L^\infty},
\]
so $u_r^\kappa\in L^\infty(\R^d)$.
For each distributional derivative,
\[
\partial_j u_r^\kappa=(\partial_j\kappa)*r.
\]
Indeed, for $\varphi\in C_c^\infty(\R^d)$, Fubini's theorem and the definition of the weak derivative $\partial_j\kappa$ give
\[
\begin{aligned}
-\int_{\R^d}u_r^\kappa(x)\partial_j\varphi(x)\,dx
&=-\int_{\R^d}\int_{\R^d}
\kappa(x-y)\partial_j\varphi(x)\,dx\,r(dy)\\
&=\int_{\R^d}\int_{\R^d}
\partial_j\kappa(x-y)\varphi(x)\,dx\,r(dy).
\end{aligned}
\]
Furthermore, Tonelli's theorem gives
\[
\|(\partial_j\kappa)*r\|_{L^1}
\le
\int_{\R^d}\int_{\R^d}
|\partial_j\kappa(x-y)|\,dx\,r(dy)
=
\|\partial_j\kappa\|_{L^1}.
\]
Hence $u_r^\kappa\in W^{1,1}(\R^d)$, and in particular
$u_r^\kappa\in W^{1,1}_{\mathrm{loc}}(\R^d)$.
Since $\kappa>0$ and $r(\R^d)=1$, we also have $u_r^\kappa(x)>0$ for every $x$.

Fix $\alpha$.
The functions $\eta^{[\alpha]}$ and $\nabla_\alpha\eta^{[\alpha]}$ likewise belong to $C_0(\R^d)\cap L^1(\R^d)\cap L^\infty(\R^d)$.
This follows from
\[
\eta^{[\alpha]}(z)
=
\eta_\alpha(z^{(\alpha)})
\prod_{\beta\ne\alpha}\kappa_\beta(z^{(\beta)}),
\qquad
\nabla_\alpha\eta^{[\alpha]}(z)
=
\nabla_\alpha\eta_\alpha(z^{(\alpha)})
\prod_{\beta\ne\alpha}\kappa_\beta(z^{(\beta)})
\]
and the corresponding $C_0,L^1,L^\infty$ properties of the factors.
Therefore
\[
\Psi_{\alpha,r}^\kappa=\eta^{[\alpha]}*r,
\qquad
\nabla_\alpha\Psi_{\alpha,r}^\kappa
=(\nabla_\alpha\eta^{[\alpha]})*r
\]
belong to $C_0(\R^d)\cap L^1(\R^d)\cap L^\infty(\R^d)$ by (F2) of \Cref{sec:preliminaries} and the preceding $L^\infty$ bound.
Since $\eta^{[\alpha]}>0$, we also have $\Psi_{\alpha,r}^\kappa>0$.

It remains to prove the local Sobolev regularity of $\nabla_\alpha\Psi_{\alpha,r}^\kappa$.
For each component
\[
g_i(z):=\partial_{z_i^{(\alpha)}}\eta^{[\alpha]}(z),
\qquad i=1,\ldots,d_\alpha,
\]
we show that $g_i\in W^{1,1}_{\mathrm{loc}}(\R^d)$ and that the integrals of its weak derivatives are uniformly bounded over all translates of a compact set.
More precisely, let $K\Subset\R^d$ and choose compact rectangles $K_1,\ldots,K_M$ with $K\subset K_1\times\cdots\times K_M$.
For derivatives in the $\alpha$-th block,
\[
\sup_{y\in\R^d}
\int_{K-y}
\left|
\nabla_\alpha^2\eta_\alpha(z^{(\alpha)})
\prod_{\beta\ne\alpha}\kappa_\beta(z^{(\beta)})
\right|\,dz
\le
\|\nabla_\alpha^2\eta_\alpha\|_{L^\infty}|K_\alpha|
\prod_{\beta\ne\alpha}\|\kappa_\beta\|_{L^1},
\]
and for derivatives in a block $\beta\ne\alpha$,
\[
\begin{aligned}
&\sup_{y\in\R^d}
\int_{K-y}
\left|
\nabla_\alpha\eta_\alpha(z^{(\alpha)})
\otimes\nabla_\beta\kappa_\beta(z^{(\beta)})
\prod_{\gamma\ne\alpha,\beta}\kappa_\gamma(z^{(\gamma)})
\right|\,dz\\
&\qquad\le
\|\nabla_\alpha\eta_\alpha\|_{L^1}
\|\nabla_\beta\kappa_\beta\|_{L^1}
\prod_{\gamma\ne\alpha,\beta}\|\kappa_\gamma\|_{L^1},
\end{aligned}
\]
where $|K_\alpha|$ in the first estimate denotes the Lebesgue measure of $K_\alpha$.
Now take $\varphi\in C_c^\infty(K)$.
By Fubini's theorem and the definition of the weak derivatives of $g_i$,
\[
-\int_{\R^d}(g_i*r)(x)\partial_j\varphi(x)\,dx
=
\int_{\R^d}\int_{\R^d}D_j g_i(x-y)\varphi(x)\,dx\,r(dy),
\]
so $D_j(g_i*r)=(D_j g_i)*r$ in $\mathcal D'(K)$.
The translate-uniform bounds above and Fubini's theorem give
\[
\int_K |(D_j g_i*r)(x)|\,dx
\le
\int_{\R^d}\int_{K-y}|D_j g_i(z)|\,dz\,r(dy)<\infty.
\]
Hence $g_i*r\in W^{1,1}(K)$.
Since $K$ was arbitrary,
\[
\nabla_\alpha\Psi_{\alpha,r}^\kappa
\in W^{1,1}_{\mathrm{loc}}(\R^d;\R^{d_\alpha}).
\]

Classical differentiation in the $\alpha$-th block also commutes with the convolution.
Indeed, $\nabla_\alpha\eta^{[\alpha]}$ and $\nabla_\alpha^2\eta^{[\alpha]}$ are bounded continuous functions and $r$ is a finite measure.
Dominated convergence therefore gives
\[
\nabla_\alpha\Psi_{\alpha,r}^\kappa
=(\nabla_\alpha\eta^{[\alpha]})*r,
\qquad
\nabla_\alpha^2\Psi_{\alpha,r}^\kappa
=(\nabla_\alpha^2\eta^{[\alpha]})*r.
\]
The continuity of these two convolutions is verified as follows.
Let $h$ be either of the two kernels above and let $x_n\to x$.
If $h\equiv0$ there is nothing to prove, so assume $\|h\|_{L^\infty}>0$.
Given $\varepsilon>0$, choose $R>0$ such that
\[
r(B_R^c)<\frac{\varepsilon}{4\|h\|_{L^\infty}}.
\]
For all sufficiently large $n$ we have $x_n\in\overline{B_1(x)}$, and $h$ is uniformly continuous on the compact set
\[
\{w-y:w\in\overline{B_1(x)},\ y\in B_R\}.
\]
Hence, taking $n$ large enough,
\[
|h(x_n-y)-h(x-y)|<\varepsilon/2
\qquad\text{for all }y\in B_R.
\]
The contribution on $B_R^c$ is bounded by
\[
\int_{B_R^c}|h(x_n-y)-h(x-y)|\,r(dy)
\le
2\|h\|_{L^\infty}r(B_R^c)
<
\varepsilon/2.
\]
Therefore
\[
\int |h(x_n-y)-h(x-y)|\,r(dy)<\varepsilon,
\]
and the convolution is continuous.

Finally, we prove the block identities.
Convolving the first identity in \eqref{eq:block-kernel-identities} with $r$ gives
\[
\nabla_\alpha\Psi_{\alpha,r}^\kappa(x)
=
\int\nabla_\alpha\eta^{[\alpha]}(x-y)\,r(dy)
=
\frac1{c_{1,\alpha}}
\int (y^{(\alpha)}-x^{(\alpha)})\kappa(x-y)\,r(dy),
\]
which is \eqref{eq:block-grad-id}.
In particular,
\[
m_r^{\kappa,(\alpha)}(x)
=
x^{(\alpha)}u_r^\kappa(x)
+
c_{1,\alpha}\nabla_\alpha\Psi_{\alpha,r}^\kappa(x),
\]
so $m_r^\kappa$ is continuous.
Since $u_r^\kappa$ is a positive continuous function, $a_r^\kappa=m_r^\kappa/u_r^\kappa$ is continuous as well.

Convolving the second identity in \eqref{eq:block-kernel-identities} and using the commutation of the $\alpha$-block derivatives established above gives \eqref{eq:block-elliptic-id}.
Finally, \eqref{eq:block-barycenter-id} follows from \eqref{eq:block-grad-id}, \eqref{eq:block-elliptic-id}, and $u_r^\kappa>0$.
\end{proof}

\subsection{Proof of field identifiability}

The proof runs the argument of \Cref{thm:companion-elliptic} blockwise. Two points are genuinely new: the regularity and integrability checks now concern velocity fields supported on a single block, and the Wronskian argument fixes the ratio of the potentials only along each block fiber, so a final patching step across blocks is required.

\begin{proof}[Proof of \Cref{thm:block-separable-ce-identifiability}]
By \Cref{lem:block-ce-identities}, $u_p^\kappa,u_q^\kappa$ are positive continuous functions.
Moreover, $a_p^\kappa,a_q^\kappa$ are defined on all of $\R^d$, so $V_{p,q}^\kappa$ is globally well defined.
Assume now $V_{p,q}^\kappa\equiv0$.
Then $a_p^\kappa=a_q^\kappa=:a$.

Fix one block $\alpha$ and set $b_\alpha(x):=(a^{(\alpha)}(x)-x^{(\alpha)})/(c_{1,\alpha}c_{2,\alpha})$.
For $r=p,q$, \eqref{eq:block-barycenter-id} and \eqref{eq:block-elliptic-id} yield
\begin{equation}\label{eq:block-first-order-relation}
\nabla_\alpha\Psi_{\alpha,r}^\kappa
=b_\alpha
(\lambda_{0,\alpha}I-\lambda_{1,\alpha}\Delta_\alpha)
\Psi_{\alpha,r}^\kappa.
\end{equation}
To lighten the notation, we write $\Psi_{\alpha,r}:=\Psi_{\alpha,r}^\kappa$ for the rest of the proof.

Define the block Wronskian field
\[
J_\alpha:=
\Psi_{\alpha,q}\nabla_\alpha\Psi_{\alpha,p}
-
\Psi_{\alpha,p}\nabla_\alpha\Psi_{\alpha,q}.
\]
First suppose $\lambda_{1,\alpha}=0$.
Multiplying the $p$-equation in \eqref{eq:block-first-order-relation} by $\Psi_{\alpha,q}$ and the $q$-equation by $\Psi_{\alpha,p}$ and subtracting gives $J_\alpha=0$.

Now suppose $\lambda_{1,\alpha}>0$.
Since $\Psi_{\alpha,p}$ and $\Psi_{\alpha,q}$ are $C^2$ in the $\alpha$-th block variable,
\[
s_\alpha:=\nabla_\alpha\cdot J_\alpha
=
\Psi_{\alpha,q}\Delta_\alpha\Psi_{\alpha,p}
-
\Psi_{\alpha,p}\Delta_\alpha\Psi_{\alpha,q}
\]
is defined pointwise.
Multiplying and subtracting the two equations of \eqref{eq:block-first-order-relation} as above gives
\[
J_\alpha
=
-\lambda_{1,\alpha}b_\alpha
\left(
\Psi_{\alpha,q}\Delta_\alpha\Psi_{\alpha,p}
-
\Psi_{\alpha,p}\Delta_\alpha\Psi_{\alpha,q}
\right)
=
-\lambda_{1,\alpha}b_\alpha s_\alpha,
\]
and combining this with the definition $s_\alpha=\nabla_\alpha\cdot J_\alpha$ yields
\begin{equation}\label{eq:block-stationary-s}
s_\alpha+
\lambda_{1,\alpha}\nabla_\alpha\cdot(b_\alpha s_\alpha)=0
\qquad\text{in }\mathcal D'(\R^d).
\end{equation}

By \eqref{eq:block-elliptic-id},
\[
\Delta_\alpha\Psi_{\alpha,r}
=
\lambda_{1,\alpha}^{-1}
\left(\lambda_{0,\alpha}\Psi_{\alpha,r}
-c_{2,\alpha}u_r^\kappa\right)
\qquad(r=p,q),
\]
and therefore
\[
s_\alpha=
\frac{c_{2,\alpha}}{\lambda_{1,\alpha}}
\left(
\Psi_{\alpha,p}u_q^\kappa-
\Psi_{\alpha,q}u_p^\kappa
\right).
\]
By \Cref{lem:block-ce-identities},
$\Psi_{\alpha,p},\Psi_{\alpha,q},u_p^\kappa,u_q^\kappa$ are continuous and belong to $L^1(\R^d)\cap L^\infty(\R^d)$.
Hence, for instance,
\[
\|\Psi_{\alpha,p}u_q^\kappa\|_{L^1}
\le
\|\Psi_{\alpha,p}\|_{L^\infty}\|u_q^\kappa\|_{L^1},
\qquad
\|\Psi_{\alpha,p}u_q^\kappa\|_{L^\infty}
\le
\|\Psi_{\alpha,p}\|_{L^\infty}\|u_q^\kappa\|_{L^\infty},
\]
and the other term is handled in the same way.
Therefore $s_\alpha\in L^1(\R^d)\cap L^\infty(\R^d)$ and $s_\alpha$ is continuous.
Moreover,
\[
\|J_\alpha\|_{L^1}
\le
\|\Psi_{\alpha,q}\|_{L^\infty}
\|\nabla_\alpha\Psi_{\alpha,p}\|_{L^1}
+
\|\Psi_{\alpha,p}\|_{L^\infty}
\|\nabla_\alpha\Psi_{\alpha,q}\|_{L^1},
\]
so $J_\alpha\in L^1(\R^d;\R^{d_\alpha})$.

We next verify the Sobolev regularity of $b_\alpha$.
From \eqref{eq:block-elliptic-id} and \eqref{eq:block-first-order-relation},
\[
b_\alpha
=
\frac{\nabla_\alpha\Psi_{\alpha,q}}
{(\lambda_{0,\alpha}I-\lambda_{1,\alpha}\Delta_\alpha)
\Psi_{\alpha,q}}
=
\frac{\nabla_\alpha\Psi_{\alpha,q}}
{c_{2,\alpha}u_q^\kappa}.
\]
Let $K\Subset\R^d$ be compact.
Since $u_q^\kappa$ is positive and continuous,
\[
m_K:=\inf_{x\in K}u_q^\kappa(x)>0.
\]
Choose $\theta\in C^\infty(\R)$, with $\theta$ and $\theta'$ bounded, such that $\theta(t)=1/t$ for $t\ge m_K/2$.
Since $u_q^\kappa\in W^{1,1}(K)\cap L^\infty(K)$,
\[
D(\theta(u_q^\kappa))
=
\theta'(u_q^\kappa)D u_q^\kappa\in L^1(K),
\]
so $\theta(u_q^\kappa)\in W^{1,1}(K)\cap L^\infty(K)$.
Moreover, $v:=\nabla_\alpha\Psi_{\alpha,q}$ belongs to
$W^{1,1}(K;\R^{d_\alpha})\cap L^\infty(K;\R^{d_\alpha})$, so
\[
D\bigl(v\,\theta(u_q^\kappa)\bigr)
=
\theta(u_q^\kappa)Dv
+
v\,\theta'(u_q^\kappa)D u_q^\kappa
\in L^1(K).
\]
Hence
\[
b_\alpha|_K
=
\frac1{c_{2,\alpha}}
\nabla_\alpha\Psi_{\alpha,q}\,\theta(u_q^\kappa)
\in W^{1,1}(K;\R^{d_\alpha}).
\]
Since $K$ was arbitrary,
$b_\alpha\in W^{1,1}_{\mathrm{loc}}(\R^d;\R^{d_\alpha})$.

Define $B_\alpha$ as the vector field on $\R^d$ whose $\alpha$-th block component is $b_\alpha$ and whose remaining block components vanish.
Then $B_\alpha\in W^{1,1}_{\mathrm{loc}}(\R^d;\R^d)$, and \eqref{eq:block-stationary-s} is equivalent to
\[
s_\alpha+
\lambda_{1,\alpha}\nabla\cdot(B_\alpha s_\alpha)=0
\qquad\text{in }\mathcal D'(\R^d).
\]
Moreover, the only nonzero block of $B_\alpha s_\alpha$ is $b_\alpha s_\alpha$, and the identity $J_\alpha=-\lambda_{1,\alpha}b_\alpha s_\alpha$ implies
\[
B_\alpha s_\alpha
=
-\lambda_{1,\alpha}^{-1}\widetilde J_\alpha,
\]
where $\widetilde J_\alpha\in L^1(\R^d;\R^d)$ places $J_\alpha$ in the $\alpha$-th block and zero in the remaining blocks.
In particular, $B_\alpha s_\alpha\in L^1(\R^d;\R^d)$.

Now, as in the proof of \Cref{thm:companion-elliptic}, set $\rho(t,x):=e^{t/\lambda_{1,\alpha}}s_\alpha(x)$. Then
\[
\partial_t\rho+\nabla\cdot(B_\alpha\rho)=0
\qquad\text{in }\mathcal D'(\R_t\times\R^d).
\]
Since $s_\alpha\in L^1\cap L^\infty$,
\[
\rho\in
L^\infty_{\mathrm{loc}}
\bigl(\R_t;L^1(\R^d)\cap L^\infty(\R^d)\bigr),
\]
and
\[
B_\alpha\rho
=
-\frac{e^{t/\lambda_{1,\alpha}}}{\lambda_{1,\alpha}}
\widetilde J_\alpha
\in
L^\infty_{\mathrm{loc}}
\bigl(\R_t;L^1(\R^d;\R^d)\bigr).
\]
Therefore \Cref{thm:renorm-sobolev,lem:renorm-mass} apply, and
$t\mapsto\|\rho(t)\|_{L^1}$ is constant in the distributional sense.
But $\|\rho(t)\|_{L^1}=e^{t/\lambda_{1,\alpha}}\|s_\alpha\|_{L^1}$, and the right-hand side can be distributionally constant in $t$ only if $\|s_\alpha\|_{L^1}=0$.
Hence $s_\alpha=0$ a.e., and by continuity $s_\alpha\equiv0$.
The identity $J_\alpha=-\lambda_{1,\alpha}b_\alpha s_\alpha$ then gives $J_\alpha\equiv0$.

We have thus proved $J_\alpha\equiv0$ both when $\lambda_{1,\alpha}=0$ and when $\lambda_{1,\alpha}>0$.
Since $\Psi_{\alpha,q}>0$ and $\Psi_{\alpha,p},\Psi_{\alpha,q}$ are $C^1$ in the $\alpha$-th block variable,
\[
\nabla_\alpha
\left(
\frac{\Psi_{\alpha,p}}{\Psi_{\alpha,q}}
\right)
=
\frac{J_\alpha}{\Psi_{\alpha,q}^2}
=0.
\]
Each $x^{(\alpha)}$-fiber is connected, so this ratio is constant on each fiber.
Consequently, there is a positive function $C_\alpha$ such that
\[
\Psi_{\alpha,p}(x)
=
C_\alpha(x^{(\widehat\alpha)})\Psi_{\alpha,q}(x),
\]
where
\[
x^{(\widehat\alpha)}
:=(x^{(1)},\ldots,x^{(\alpha-1)},x^{(\alpha+1)},\ldots,x^{(M)})
\]
denotes the variables outside the $\alpha$-th block.
Since $C_\alpha$ is independent of $x^{(\alpha)}$, for fixed $x^{(\widehat\alpha)}$,
\[
(\lambda_{0,\alpha}I-\lambda_{1,\alpha}\Delta_\alpha)
\bigl(C_\alpha\Psi_{\alpha,q}\bigr)
=
C_\alpha
(\lambda_{0,\alpha}I-\lambda_{1,\alpha}\Delta_\alpha)
\Psi_{\alpha,q}.
\]
Applying \eqref{eq:block-elliptic-id} gives
\[
u_p^\kappa(x)
=
C_\alpha(x^{(\widehat\alpha)})u_q^\kappa(x).
\]
Equivalently, the continuous function
\[
h(x):=\frac{u_p^\kappa(x)}{u_q^\kappa(x)}
\]
is independent of $x^{(\alpha)}$.
This conclusion holds for every $\alpha=1,\ldots,M$.
Given two points $x,y\in\R^d$, following a finite path from $x$ to $y$ that changes one block at a time gives $h(x)=h(y)$.
Hence $h$ is constant on $\R^d$, and for some $c>0$,
\[
u_p^\kappa=c u_q^\kappa.
\]

Integrating both sides in $x$ and applying Tonelli's theorem, we obtain, for every $r\in\mathcal P(\R^d)$,
\[
\begin{aligned}
\int_{\R^d}u_r^\kappa(x)\,dx
&=\int_{\R^d}\int_{\R^d}\kappa(x-y)\,r(dy)\,dx\\
&=\int_{\R^d}\int_{\R^d}\kappa(z)\,dz\,r(dy)
=\int_{\R^d}\kappa(z)\,dz.
\end{aligned}
\]
Hence $c=1$ and $\kappa*(p-q)=0$.
Taking Fourier transforms yields
\[
\widehat\kappa(\xi)\bigl(\widehat p(\xi)-\widehat q(\xi)\bigr)=0.
\]
By the product structure and Fubini's theorem,
\[
\widehat\kappa(\xi)
=
\prod_{\alpha=1}^M
\widehat{\kappa_\alpha}(\xi^{(\alpha)}).
\]
Each factor $\kappa_\alpha$ satisfies the scalar companion-elliptic assumption in its block dimension, so by \Cref{lem:ce-spectral-law} each $\widehat{\kappa_\alpha}$ is strictly positive everywhere.
Hence $\widehat\kappa(\xi)>0$ for all $\xi\in\R^d$, and Fourier uniqueness for the finite signed measure $p-q$, recalled in \Cref{sec:preliminaries}, gives $p=q$.
\end{proof}

\end{document}